\documentclass[11pt]{article}
\usepackage[margin = 1in]{geometry}

\usepackage{graphicx}
\usepackage{caption}
\usepackage{subcaption}
\usepackage{afterpage}

\usepackage{amsmath,amssymb,amsthm,hyperref,color}
\usepackage[inline]{enumitem}
\usepackage{algorithm}
\usepackage{algpseudocode}
\hypersetup{colorlinks=true,citecolor=blue, linkcolor=blue,
urlcolor=blue}

\newtheorem{theorem}{Theorem}
\newtheorem{lemma}[theorem]{Lemma}
\newtheorem{corollary}[theorem]{Corollary}

\newtheorem{claim}[theorem]{Claim}

\begin{document} 
\title{Rapid Mixing Swendsen-Wang Sampler for Stochastic Partitioned Attractive Models}

\author{Sejun Park\thanks{School of Electrical Engineering, Korea Advanced Institute of Science Technology, Republic of Korea.
\texttt{\{sejun.park,cirdan,jinwoos\}@kaist.ac.kr}. This work was supported by Institute for Information \& communications Technology Promotion (IITP) grant funded by the Korea government (MSIP) (NO.R0132-17-1005), Content visual browsing technology in the online and offline environments. 
Sejun Park was supported in part by the Bloomberg Data Science Research Grant.}
\and Yunhun Jang$^*$\hspace{-0.05in}
\and
Andreas Galanis\thanks{University of Oxford,
 Wolfson Building, Parks Road, Oxford, OX1~3QD, UK.
  \texttt{andreas.galanis@cs.ox.ac.uk}. The research leading to these results has received funding from the European Research Council under
the European Union's Seventh Framework Programme (FP7/2007-2013) ERC grant agreement no. 334828. The paper
reflects only the authors' views and not the views of the ERC or the European Commission. The European Union is not liable for any use that may be made of the information contained therein.}
\and Jinwoo Shin$^*$
\and
 Daniel \v{S}tefankovi\v{c}\thanks{Department of Computer Science, University of Rochester, Rochester, NY 14627.  Research
supported by NSF grant CCF-0910415.}
 \and 
Eric Vigoda\thanks{School of Computer Science, Georgia
Institute of Technology, Atlanta, GA 30332.
 \texttt{vigoda@cc.gatech.edu}. Research supported in part by NSF grant CCF-1217458.}
}

\maketitle
\begin{abstract} 
The Gibbs sampler is a particularly popular Markov chain used
for learning and inference problems in Graphical Models (GMs). These tasks
are computationally intractable in general, and the Gibbs sampler often
suffers from slow mixing. 
In this paper, we study the Swendsen-Wang dynamics which is a more sophisticated
Markov chain designed to overcome bottlenecks that impede the Gibbs sampler.
We prove $O(\log{n})$ mixing time for attractive binary pairwise GMs (i.e., ferromagnetic Ising models)
on stochastic partitioned graphs having $n$ vertices, under some mild conditions, 
including low temperature regions where 
the Gibbs sampler provably mixes exponentially slow. Our experiments also confirm that
the Swendsen-Wang sampler significantly outperforms the Gibbs sampler when they are used for learning parameters of 
attractive GMs.
\end{abstract}

\section{Introduction}

Graphical models (GMs) express a factorization of joint multivariate probability distributions in statistics via a graph of relations
between variables. GMs have been used successfully in information theory \cite{gallager1962low}, statistical physics \cite{baxter2007exactly}, artificial
intelligence \cite{pearl2014probabilistic} and machine learning \cite{jordan1998learning}. 
For typical learning and inference problems using GMs, 
marginalizing the joint distribution, or equivalently computing the partition function (normalization factor), is
the key computational bottleneck; this sampling/counting problem is computationally
intractable in general, more formally, it is NP-hard even to approximate the partition function \cite{cooper1990computational,roth1996hardness}.
Nevertheless, Markov Chain Monte Carlo (MCMC) methods, typically using the Gibbs sampler, are widely-used 
in learning and inference applications of GM, but they often suffer from slow mixing.

To address the potential slow mixing of the Gibbs sampler, there have been extensive efforts in the literature to establish fast mixing regimes of the Gibbs sampler (also known as the Glauber dynamics).  Most of these theoretical works have studied under various perspectives the Ising model and its variants \cite{mossel2009rapid,levin2010glauber,cuff2012glauber}. 
Given a graph $G=(V,E)$ having $n$ vertices and
parameters $\beta=[\beta_{uv}:(u,v)\in E]\in\mathbb{R}^{|E|},\gamma=[\gamma_v:v\in V]\in\mathbb{R}^{n}$,
the Ising model is a joint probability distribution on all spin configurations
$\Omega=\{\sigma\,:\,\sigma=[\sigma_v]\in\{-1,1\}^{n}\}$ such that
\begin{equation}\label{eq:isingdist}
\mu(\sigma)
\propto\exp\bigg(\sum_{(u,v)\in E}\beta_{uv}\sigma_u\sigma_v
+\sum_{v\in V}\gamma_v\sigma_v\bigg).
\end{equation}
The parameter $\gamma$ corresponds to the presence of an ``external (magnetic) field'', and
when $\gamma_v=0$ for all $v\in V$, we say the model has no (or zero) external field.
If $\beta_{uv}\geq0$ for all $(u,v)\in E$ the model is called {\em ferromagnetic/attractive}, and {\em anti-ferromagnetic/repulsive}
if $\beta_{uv}\leq0$ for all $(u,v)\in E$. 
It is naturally expected that the Gibbs sampler mixes slow if interaction strengths of GM are high, i.e., $\beta$ is large which corresponds to low temperature regimes.
For example, for the ferromagnetic Ising model on the complete graph $G$ (which is commonly referred to 
as the mean-field model)
it is known that the mixing-time in the high temperature regime ($\beta<1$) is $O(n \log n)$, 
whereas the mixing-time in the low temperature regime ($\beta>1$) is exponential in $n$ \cite{levin2010glauber}.

This paper focuses on ferromagnetic Ising models (FIM), where any pairwise binary attractive GM can be expressed by FIM.  We study the Swendsen-Wang dynamics\footnote{The Swendsen-Wang dynamics is formally defined in Section \ref{sec:SW}.} which is a more sophisticated
Markov chain designed to overcome bottlenecks that impede the Gibbs sampler.
Pairwise binary attractive GMs, equivalently FIMs, have gained much attention in the GM literature
because they do not contain frustrated cycles and have several advantages to design good algorithms
for approximating the partition function \cite{jerrum1993polynomial,weller2013bethe,weller2014approximating,nobuyuki2006convergence,mooij2007sufficient,ruozzi2014making}.
Furthermore, they have been used for various machine learning applications.
For example, the non-negative Boltzmann machine (NNBM) has been used to describe multimodal non-negative data \cite{downs2000nonnegative}.
Moreover, the non-negative restricted Boltzmann machine (RBM), 
which is equivalent to
FIM on complete bipartite graphs, has been studied
in the context of unsupervised deep learning models \cite{nguyen2013learning}, where non-negativity (i.e., ferromagneticity) provides non-negative matrix factorization \cite{lee1999learning} like interpretable features, which is especially useful for analyzing medical data \cite{tran2015learning, li2015bone} and document data \cite{nguyen2013learning}.
FIM is also a popular model for studying strategic diffusion in social networks \cite{ok2014maximizing,montanari2010spread},
where in this case $\beta_{uv}$ represents a friendship or other positive relationships between two individuals $u,v$.

Motivated by the recent studies on FIM,
we prove $O(\log{n})$ mixing time of the Swendsen-Wang sampler for FIM 
on stochastic partitioned 
graphs\footnote{See Section \ref{sec:main} for the formal definition of stochastic partitioned graphs},
which include complete bipartite graphs
and social network models (e.g., stochastic block models \cite{holland1983stochastic}) as special cases.
In particular, we show that the Swendsen-Wang chain mixes fast in
low temperature regions where 
the Gibbs sampler provably mixes exponentially slow. 
Our experimental results also confirm that 
the Swendsen-Wang sampler significantly outperforms the Gibbs sampler for learning parameters of attractive GMs.
We remark that it has been recently shown that an arbitrary binary pairwise GM
can be approximated by an FIM of a certain partitioned structure.
In conjunction with this, we believe that our results  potentially extend to a certain class of non-attractive GMs as well (see Section \ref{conc}).

\vspace{0.1in}
\noindent {\bf Related work.} 
There has been considerable effort on analyzing the mixing times of the 
Swendsen-Wang and Gibbs samplers
for the ferromagnetic Ising model.  
All of the below theoretical works consider `uniform' parameters on edges, i.e., all $\beta_{uv}$'s are equal, and zero external field, i.e., $\gamma_v=0$.
There are several works showing examples where the Swendsen-Wang dynamics has
exponentially slow mixing time \cite{GJ,CF,BCFKTVV,BCT,GSVY} 
for the Potts model which is the generalization of
the Ising model to more than two spins; all of these slow mixing results are at the critical point
for the associated phase transition.  For the Ising model, 
it was very recently shown that the Swendsen-Wang dynamics is rapidly mixing on every graph and at every (positive) temperature \cite{guo2017random}; the mixing time is a large polynomial, e.g., $O(n^{10})$ for complete bipartite graphs, so this general result does not give bounds which are useful in practice.
However, the appeal for utilizing this dynamics is that
its mixing time is conjectured to be much smaller, and we prove an $O(\log n)$ bound  for stochastic partitioned graphs.
It is conjectured that the mixing time of the Swendsen-Wang dynamics is a small polynomial or $O(\log n)$, 
this is part of the appeal for utilizing this dynamics.

For the mean-field model (i.e., the complete graph) a detailed analysis of the
Swendsen-Wang dynamics was established 
by \cite{long2014power} who proved that 
the mixing time is $\Theta(1)$ for $\beta<\beta_c$, $O(n^{1/4})$ for $\beta=\beta_c$ and $O(\log n)$ for $\beta>\beta_c$ where $\beta_c$ is the inverse critical temperature.
For the two-dimensional lattice, \cite{ullrich2012rapid} established polynomial mixing time of the Swendsen-Wang dynamics
for all $\beta>0$.
On the other hand,
the mixing time of the Gibbs sampler
(also known as the Glauber dynamics or the Metropolis-Hastings algorithm) 
for the complete graph is known to be 
$\Theta(n\log n)$ for $\beta<\beta_c$, $\Theta(n^{3/2})$ for $\beta=\beta_c$ and $e^{\Omega(n)}$ for $\beta>\beta_c$ \cite{levin2010glauber}.
For the Erd\H{o}s-R\'enyi random graph $G(n,d/n)$,
the mixing time of the Gibbs chain is
$O(n^{1+\Theta(1/\log\log n)})$ for $d\tanh\beta<1$ \cite{mossel2013exact} and $e^{\Omega(n)}$ for $d\tanh\beta>1$ \cite{gerschcnfeld2007reconstruction} with high probability over the choice of the graph.

\section{Preliminaries} 

\subsection{Swendsen-Wang Sampler}
\label{sec:SW}

\begin{figure*}[t]
    \centering
        \includegraphics[width=\textwidth]{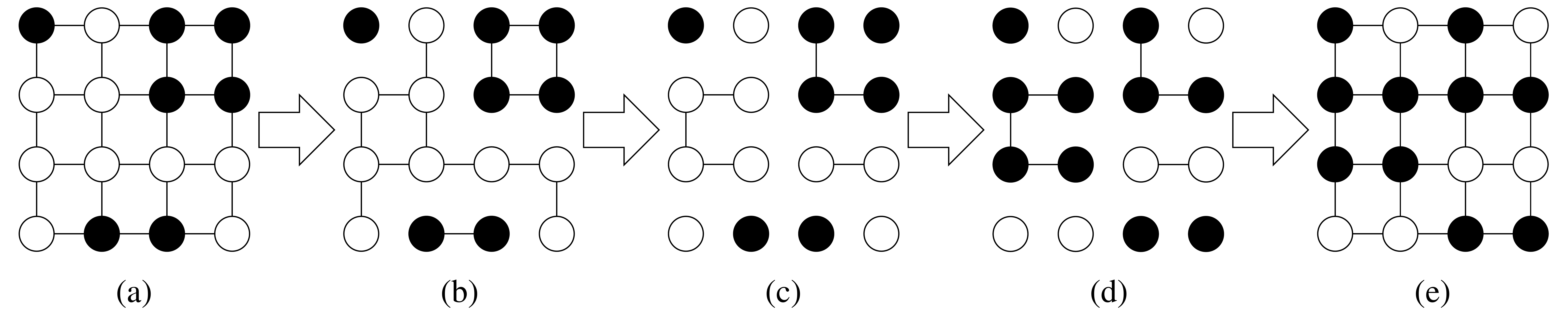}
    \caption{Illustration of a single iteration of the Swendsen-Wang dynamics. Each subfigure represents (a) an input $X_t$ (b) a subgraph induced by the set of monochromatic edges $M$ (c) a subgraph induced by the set of monochromatic edges $M^\prime$ after the step 2 (d) a configuration after the step 3 (e) an output $X_{t+1}$ where black and white imply assignments $-1,+1$ respectively.}\label{fig:sw}
\end{figure*}

The Swendsen-Wang dynamics  \cite{swendsen1987nonuniversal} 
is a Markov chain $\{X_t\in \Omega :t=0,1,2,\dots\}$
whose stationary (i.e., invariant) distribution is the distribution $\mu$ in \eqref{eq:isingdist}.
A step of the
Swendsen-Wang dynamics works at a high-level as follows:
(i) the current spin configuration $X_t$ is converted
into a configuration $M$ in the random-cluster model \cite{edwards1988generalization}
by taking the set of the monochromatic edges in the spin configuration, (ii) then we do a
percolation step on $M$ where each edge is deleted with
some probability, and finally (iii) each connected component
of the percolated subgraph chooses a random spin; 
this yields the new spin configuration $X_{t+1}$. Whereas
the traditional Gibbs sampler modifies the spin at one
vertex in a step, the Swendsen-Wang dynamics may
change the spin at every vertex in a single step.
For ferromagnetic Ising models with no external field, the transition  
from $X_t$ to $X_{t+1}$ is defined as follows:
\begin{itemize}
\item[1.] Let $M$ be the set of monochromatic edges in $X_t$, i.e.,
$M = \{(u,v)\in E: X_t(u) = X_t(v)\}$.
\item[2.] For each edge $(u,v)\in M$, delete it with probability $1-p_{uv}$,
where $p_{uv}=1-\exp(-2\beta_{uv})$.
 Let $M^\prime$ denote the set of monochromatic edges that were not deleted.
\item[3.] For each connected component $C$ of the subgraph $G^\prime=(V,M^\prime)$, 
independently, choose a spin $s\in\{-1,+1\}$ uniformly at random and
assign spin $s$ to all vertices in $C$.
Let $X_{t+1}$ denote the resulting spin configuration.
\end{itemize}
One can generalize the dynamics 
to a model having external fields by modifying step 3 as follows \cite{albeverio1992ideas}:
\begin{itemize}
\item[3.] For each connected component $C$ of the subgraph $G^\prime=(V,M^\prime)$, set
\[ s=\begin{cases} +1 & \mbox{ with probability }
\frac{\exp\left(2\sum_{v\in V(C)}\gamma_v\right)}{1+\exp\left(2\sum_{v\in V(C)}\gamma_v\right)}
\\
-1 &  \mbox{ with probability } 
\frac{1}{1+\exp\left(2\sum_{v\in V(C)}\gamma_v\right)}
\end{cases}. \]
Then, assign all vertices in $C$ the chosen spin $s$ and
let $X_{t+1}$ denote the resulting spin configuration.
\end{itemize}
Figure \ref{fig:sw} visualizes each step of the Swendsen-Wang dynamics.
One can prove that
the stationary distribution of the Swendsen-Wang chain is \eqref{eq:isingdist}.

\subsection{Mixing Time and Coupling}\label{sec:coupling}

We use the following popular notion
of `mixing time':
given an ergodic Markov chain $\{X_t\in \Omega:t=0,1,2,\dots\}$ with stationary distribution $\mu$,
we define the mixing time $T_\text{mix}$ as
$$T_\text{mix}:=\min\Big\{t\, \big|\,\sup_{X_0\in\Omega,A\subset\Omega}|\Pr(X_t\in A)-\mu(A)|\le\frac{1}{4}\Big\}.$$
A classical technique for bounding the mixing time is the `coupling' technique \cite{levin2009markov}.
Consider two copies $(X_t,Y_t)$ of the same Markov chain (i.e., $X_t, Y_t$ have the same transition probabilities) defined jointly with the property that
if $X_t=Y_t$ then $X_{t^\prime}=Y_{t^\prime}$ for all $t^\prime\ge t$.
We call such $(X_t,Y_t)$ a coupling, where
$X_t,Y_t$ might be dependent and there can be many ways to design such dependencies.
Then, one can observe that
\begin{align*}
&\sup_{X_0\in\Omega,A\subset\Omega}|\Pr(X_t\in A)-\mu(A)|\\
&\le\sup_{X_0,Y_0\in\Omega,A\subset\Omega}|\Pr(X_t\in A)-\Pr(Y_t\in A)|\\
&\le\sup_{X_0,Y_0\in\Omega}\Pr(X_t\ne Y_t),
\end{align*}
which implies that 
\begin{equation}\label{eq:mixbd}
T_\text{mix}\leq \min\Big\{t\, \big|\,\sup_{X_0,Y_0\in\Omega}\Pr(X_t\ne Y_t)\leq \frac{1}{4}\Big\}.
\end{equation}
We will design a coupling for obtaining a bound on the mixing time of the Swendsen-Wang chain.

\section{Main Results}\label{sec:main}
In this section, we state the main results of this paper that 
the Swendsen-Wang chain mixes fast for a class of stochastic partitioned graphs.

We first define the notion of stochastic partitioned graphs.
Given a positive integer $r\in \mathbb Z_+$, a vector
$[\alpha_i]\in (0,1)^r$ with $\sum_i\alpha_i=1$ and a matrix $[p_{ij}]\in [0,1]^{r\times r}$,
a stochastic partitioned graph $(V,E)=G\left(n,[\alpha_i], [p_{ij}]\right)$ on $n$ vertices and $r$ partitions (or communities) of size 
$\alpha_1n,\dots,\alpha_rn$
is a random graph model such that 
$$V = \bigcup_i V_i, \quad
|V_i|=\alpha_i n\quad\mbox{and}\quad
V_i\cap V_j=\emptyset,~\mbox{for}~i\neq j.$$ 
An edge between any pair of vertices $u\in V_i,v\in V_j$ 
belongs to the graph with probability $p_{ij}$ independently. Let $E_{ij}=\{(u,v)\in E\mid u\in V_i, v\in V_j\}$.  

For example, if $p_{ii}=0$ for all $i$ and $p_{ij}=1$ for all $i\ne j$,
then the stochastic partitioned graph is the complete $r$-partite graph.
One can also check that the stochastic block model \cite{holland1983stochastic}
is a special case of the stochastic partitioned graph.
In particular, if $r=1$, $p_{11}=p$ for some $p\in [0,1]$,
we say it is the Erd\H{o}s-R\'enyi random graph and use the notation 
$G(n,p)$ to denote it.
Similarly,  
if $r=2$, $p_{11}=p_{22}=0$ and $p_{12}=p_{21}=p$ for some $p\in [0,1]$,
we say it is the bipartite Erd\H{o}s-R\'enyi random graph and use the notation 
$G(n,m,p)=(V_L,V_R,E)$ to denote it, where $n,m$ are the sizes of the parts $V_L,V_R$.
We say a graph $(V,E)$ has size $n$ if $|V|=n$ and a bipartite graph $(V_L,V_R,E)$ has size $(n,m)$ if $|V_L|=n$ and $|V_R|=m$.

\subsection{\texorpdfstring{$O(\log n)$}{Lg} Mixing in Low Temperatures}
We first establish the following rapid mixing property
of the Swendsen-Wang chain in low temperature regimes, i.e., when $\beta_{uv}=\Omega(1)$.
These are in particular the most interesting regimes since
the Gibbs chain (provably) mixes slower as $\beta_{uv}$ grows. Moreover,
these regimes are also reasonable in practical applications. For example,
in social networks, $\beta_{uv}$ represents a positive interaction strength 
between two individuals $u,v$ and it is independent of the network size $n$.

\begin{theorem}\label{thm:main2} 
The mixing time $T_\text{mix}$ of the Swendsen-Wang chain on the graph $G\left(n,[\alpha_i], [p_{ij}]\right)$
is 
$$T_\text{mix}=O(\log n)$$
with probability $1-\exp(-\Omega(n))$ over the choice of the graph if 
\begin{itemize}
\item[$\circ$] 
$\alpha_i=\Omega(1)$ for all $i$,
\item[$\circ$] 
$\gamma_v\ge0$ for all $v\in V$ (or $\gamma_v\le0$ for all $v\in V$),
\end{itemize}
and either \ref{it:cona12} or \ref{it:conb12}   holds
\begin{enumerate}[label=(\alph*),ref=(\alph*),leftmargin=*]
\item \label{it:cona12} for all $i\in [r]$, $p_{ii}=\Omega(1)$ and $\beta_{uv}=\Omega(1)$ for $(u,v)\in E_{ii}$.
\item \label{it:conb12} $p_{ij}=\Omega(1)$ and $\beta_{uv}=\Omega(1)$ for all $(u,v)\in E_{ij}$ with $i\neq j$.
\end{enumerate}
\end{theorem}
The proof of Theorem \ref{thm:main2} is presented in Section \ref{sec:pfthm:main2},
where we will show the existence of a good coupling of the Swendsen-Wang chain.
Theorem \ref{thm:main2} implies that the Swendsen-Wang chain mixes fast
as long as the positive parameters $[p_{ij}]$ and $[\beta_{uv}]$ are not `too small' (i.e., $p_{ij}=\Omega(1)$ and $\beta_{uv}=\Omega(1)$)
and all external fields $[\gamma_v]$ have the same sign (the case where all external fields $[\gamma_v]$ are negative is symmetric).
We believe that the restriction on positive external fields
is inevitable since it is known that approximating the partition function
of ferromagnetic Ising model under mixed external fields is known to be \#P-hard \cite{goldberg2007complexity}.
Despite the worst-case theoretical barrier,
the Swendsen-Wang chain still works well under mixed external fields in our experiments
(see Section \ref{sec:exp}).

\subsection{\texorpdfstring{$O(\log n)$}{Lg}  Mixing in High Temperatures}
The restriction on the parameters $[\beta_{uv}]$ in Theorem \ref{thm:main2}
is merely for technical reasons in our proof techniques, and
we believe that it is not necessary. This is because it is natural to expect that a Markov chain
mixes faster for higher temperatures.
To support the conjecture, in the following theorem, we prove 
that the Swendsen-Wang chain mixes fast even for small 
parameters $[\beta_{uv}]$ on complete bipartite graphs,  
where 
its proof is much harder than that of Theorem \ref{thm:main2}.
\begin{theorem}\label{thm:main} 
Given any constant $k>0$, the mixing time $T_\text{mix}$ of the Swendsen-Wang chain on
the complete bipartite graph $(V_L,V_R,E)$ of size $(n,kn)$   
is
$$T_\text{mix}=O(\log n)$$
if
$\beta_{uv}=-\frac{1}{2}\log\left(1-\frac{B}{n\sqrt{k}}\right)$ for all $(u,v)\in E$ for some non-negative constant $B\ne 2$ and
$\gamma_v=0$ for all $v\in V$.
\end{theorem}
Note that in the above theorem we consider the scenario $\beta_{uv}=o(1)$, i.e.,
$$\beta_{uv}=-\frac{1}{2}\log\left(1-\frac{B}{n\sqrt{k}}\right)\approx \frac{B}{2n\sqrt{k}}.$$
The proof of Theorem \ref{thm:main} is presented in
Section \ref{sec:pfthm:main}, where we will also 
show the existence of a good coupling of the Swendsen-Wang chain using a similar strategy to that in \cite{galanis2015swendsen}.
The authors of \cite{galanis2015swendsen} establish the rapid mixing property of the Swendsen-Wang chain for the complete graph
by analyzing a one-dimensional function, the so-called simplified Swendsen-Wang (see Appendix \ref{sec:simplesw}), and
utilizing known properties of Erd\H{o}s-R\'enyi random graphs.
In the case of the complete bipartite graphs,
the simplified Swendsen-Wang becomes a two-dimensional function, which makes harder to analyze.
Furthermore, the proof of Theorem \ref{thm:main} requires properties of 
the bipartite Erd\H{o}s-R\'enyi random graph $G(n,m,p)$ which
are less studied compared to the popular `non-bipartite' Erd\H{o}s-R\'enyi random graph $G(n,p)$.
In this paper, we also establish necessary properties of $G(n,m,p)$ for the proof of Theorem \ref{thm:main}. 
We believe that 
the conclusion of Theorem \ref{thm:main} holds for 
general stochastic partitioned graphs. However,
in this case, there exist technical challenges handling more randomness in graphs, and
we do not explore further in this paper.

\section{Proofs of Theorems}
\subsection{Notation}
Before we start the proof of Theorems \ref{thm:main2} and \ref{thm:main}, we first introduce some notation about configurations of the Ising model on a stochastic partitioned graph.
Given a spin configuration $\sigma$, denote by $V_-(\sigma),V_+(\sigma)$ the sets of vertices with spin $-1,+1$, respectively.
In particular, given the Ising model on a bipartite graph $(V_L, V_R,E)$ with partitions of  vertices $V_L, V_R$, edge set $E\subset\{(u,v):u\in V_L, v\in V_R\}$ and a spin configuration $\sigma\in\{-1,1\}^{|V_L\cup V_R|}$,
we say the configuration $\sigma$ has the `phase' $\alpha(\sigma)=(\alpha_L,\alpha_R)$ if the larger spin
 class of $\sigma$, say $s\in\{-,+\}$ with $V_s(\sigma)\ge(|V_L|+|V_R|)/2$,
satisfies
$$(\alpha_L,\alpha_R)=\left(\frac{V_s(\sigma)\cap V_L}{V_L},\frac{V_s(\sigma)\cap V_R}{V_R}\right).$$
One can define the induced probability on the phase $(\alpha_L,\alpha_R)$ under the Ising model as
$$\Pr(\alpha_L,\alpha_R)=\sum_{\sigma\,:\,\alpha(\sigma)=(\alpha_L,\alpha_R)}\mu(\sigma).$$

\subsection{Proof of Theorem \ref{thm:main2}}\label{sec:pfthm:main2}
In this section, we present the proof of Theorem \ref{thm:main2}. The main idea of the proof is that for every configuration $\sigma$, there is a big  connected component of roughly $n/2$ vertices which have the same spin. Crucially, the percolation step of the Swendsen-Wang chain is extremely unlikely to remove more than $O(1)$ vertices from it, since almost every cut in this component has $\Omega(n)$ edges (cf. Lemma~\ref{lem:connected2}). At the same time, at least half of the vertices of the remaining graph get the same spin as the  big component in expectation (using that $[\gamma_v]$ have the same sign). Combining these two facts, we will conclude that in $O(\log n)$ iterations of the Swendsen-Wang chain, all spins are the same with probability $\Theta(1)$ and, then, we will be able to bound the mixing time via the coupling technique.

We will focus on proving  Theorem \ref{thm:main2} when condition~\ref{it:cona12} holds (the proof under the condition~\ref{it:conb12} is almost identical). In particular, we have that the $[\gamma_v]$ have all the same sign and that there exist constants $p,\alpha,\beta>0$ such that for all $i\in [r]$, it holds that $\alpha_i\geq \alpha$, $p_{ii}\geq p$ and $\beta_{uv}\geq \beta$ for $(u,v)\in E_{ii}$. 

 We will use the following lemma for the Erd\H{o}s-R\`enyi random graph, whose proof is given in Appendix~\ref{pflem:connected2}. For a graph $G=(V,E)$ and a subset of vertices $S\subseteq V$, we denote by $\mathrm{cut}_G(S)$ the number of edges which have exactly one endpoint in $S$, and by $G[S]$ the induced subgraph of $G$ on the vertex set $S$.
\begin{lemma}\label{lem:connected2}
Let $p\in (0,1]$ be an arbitrary constant. Then, for every constant $M\geq 100/p$,  the following holds with probability $1-e^{-\Omega(n)}$ over the choice of the graph $G\sim G(n,p)$. 

Let $U$ be an arbitrary subset of vertices of $G$ with $|U|\geq n/10$. Then, for every $S\subseteq U$ such that $|S|,|U\backslash S|\geq M^2$, it holds that $\mathrm{cut}_{G[U]}(S)\geq M n$.
\end{lemma}

Let $G\sim G(n,[\alpha_i],[p_{ij}])$. Note that for $i=1,\hdots,r$ the induced subgraph $G[V_i]$ is distributed as $G(n\alpha_i,p_{ii})$, so we may apply Lemma~\ref{lem:connected2} to each $i\in [r]$ and conclude that, for $M:=\max\{100/(\alpha p), 2/\beta\}$, the following holds with probability $1-e^{-\Omega(n)}$ over the choice of the graph $G$.
\begin{equation}\label{eq:cuts}
\begin{gathered}
\mbox{$\forall i\in[r]$, $\forall U_i'\subseteq V_i$ with $|U_i'|\geq |V_i|/10$,}\\
\mbox{$\forall S\subseteq U_i'$ with $|S|, |U_i'\backslash S|\geq M^2$,}\\
\mbox{$\mathrm{cut}_{G[U_i']}(S_i)\geq M n$.}
\end{gathered}
\end{equation}
To prove the theorem, it thus suffices to show that the Swendsen-Wang chain $\{X_t:t=0,1,\dots\}$ mixes in $O(\log n)$ steps for a graph $G$ satisfying \eqref{eq:cuts}. For the rest of this section, the only assumption on the graph $G$ is \eqref{eq:cuts} and, hence, all the events and associated probabilities are with respect to the randomness of the Swendsen-Wang chain when run on the graph $G$.

At time $t$, define the spin 
$s_{i,t}$ for $i\in [r]$ as
\[s_{i,t}=\arg\max_{s\in \{-,+\}}|V_{s}(X_t)\cap V_i|,\] 
i.e., $s_{i,t}$ is the most common spin among vertices in $V_i$ at time $t$. For convenience, let $U_{i,t}:=V_{s_{i,t}}(X_{t})$ be the vertices in $V_i$ which have the spin $s_{i,t}$, so that $|U_{i,t}|\geq |V_i|/2$.   The key idea is that the  cut-property \eqref{eq:cuts} of the graph $G$ ensures that, at each step of the Swendsen-Wang chain, all but $O(1)$ vertices in $U_{i,t}$ continue to belong to $U_{i,t+1}$. 

Formally, let $M_{i,t}=\{(u,v)\in E\mid u,v\in U_{i,t}\}$ be the set of edges between vertices in $U_{i,t}$ and denote by $H=(U_{i,t},M_{i,t})$ the induced subgraph of $G$ on the set $U_{i,t}$. Let $M_{i,t}'\subseteq M_{i,t}$ be the random subset of edges which were not deleted in the percolation step of the Swendsen-Wang dynamics at time $t+1$ and consider the connected components $C^{(1)},\hdots,C^{(d)}$ of the subgraph $H'=(U_{i,t},M_{i,t}')$. For a component $C$, denote by $|C|$ the cardinality of the vertex set of the component, and by $C^*_{i,t}$ be the component with the largest size among $C^{(1)},\hdots,C^{(d)}$. We claim that for  all $t=0,1,\hdots$ and all $i\in [r]$, it holds that 
\begin{equation}\label{eq:Cstarbound}
\Pr(|C^*_{i,t}|\geq |U_{i,t}|-M^2)\geq 1-e^{-\Omega(n)}.
\end{equation}
To prove \eqref{eq:Cstarbound}, let $\mathcal{E}_t$ be the event that there exists some set $S\subseteq U_{i,t}$ with $|S|, |U_{i,t}\backslash S|\geq M^2$ such that all the edges in $\mathrm{cut}_H(S)$ were deleted in the percolation step of the Swendsen-Wang dynamics at time $t+1$. We claim that 
\[\Pr(|C^*_{i,t}|< |U_{i,t}|-M^2)\leq \Pr(\mathcal{E}_t).\]
Indeed, suppose that $|C^*_{i,t}|< |U_{i,t}|-M^2$, we will show that the event $\mathcal{E}_t$ occurs as well.  Let $S$ be the vertex set of the component $C^*_{i,t}$. Then, since $C^*_{i,t}$ is a connected component in  the percolated subgraph $H'$, we have that all the edges in $\mathrm{cut}_H(S)$ were deleted during the percolation step of the Swendsen-Wang chain at time $t+1$. If $|S|\geq M^2$, then $S$ shows that the event $\mathcal{E}_t$ occurs. Otherwise, if $|S|<M^2$, because $C^*_{i,t}$ was the largest component in $H'$, we have that there are at least $U_{i,t}/M^2\geq 2M^2$ components in $H'$ and hence the set $\widetilde{S}=C^{(1)}\cup \cdots\cup C^{( \left\lceil M^2\right\rceil)}$ satisfies $M^4\geq |\widetilde{S}|\geq M^2$ and all the edges in $\mathrm{cut}_H(\widetilde{S})$ were deleted during the percolation step of the Swendsen-Wang chain at time $t+1$. It remains to note that the probability of the event $\mathcal{E}_t$ is bounded  by
\[\sum_{\substack{S\subseteq U_{i,t};\\ |S|, |U_{i,t}\backslash S|\geq M^2}}\prod_{(u,v)\in \mathrm{cut}_H(S)}\exp(-\beta_{uv})\leq 2^ne^{-\beta M n}\leq \frac{1}{2^{n}},\] 
where we used \eqref{eq:cuts} for $U_i'=U_{i,t}$ and the bound $\beta_{uv}\geq \beta$ for all $(u,v)\in M_{i,t}$. This completes the proof of \eqref{eq:Cstarbound}.

Since all the $[\gamma_v]$ have the same sign, the probability that a vertex $v$ takes the color of $|C^*_{i,t}|$ is at least $\geq 1/2$ and hence 
\[E\big[|U_{i,t+1}|\,\big|\, |C^*_{i,t}|\big]\geq \big(|V_i|+|C^*_{i,t}|\big)/2.\] 
Now take expectations conditioned on $U_{i,t}$. By \eqref{eq:Cstarbound}, we have that with probability $1-e^{-\Omega(n)}$ it holds that $|C^*_{i,t}|\geq  |U_{i,t}|-M^2\geq |V_i|/2-M^2$ and hence we obtain that
\begin{align*}
E\big[|U_{i,t+1}|\,\big|\, |U_{i,t}|\big]\geq \big(|V_i|+|U_{i,t}|-M^2\big)/2+o(1).
\end{align*}
Thus, letting
\[N_{i,t}:=|V_i|-|U_{i,t}|\]
we obtain that  for all $t=0,1,\hdots$ and every $i\in [r]$, it holds that
\begin{equation*}
E[N_{i,t+1}\mid N_{i,t}]\leq \frac{1}{2}(N_{i,t}+2M^2).
\end{equation*}
It follows that for $T=\left\lceil 2\log n\right\rceil$, it holds that $E[N_{i,T}]\leq 4M^2$ for all $i\in [r]$, and hence by linearity of expectation we have that $E[\mbox{$\sum_{i\in r}$} N_{i,T}]\leq 4M^2r$. Thus, by Markov's inequality, we  obtain that with probability at least $1/2$, for the state $X_T$ it holds that 
\[\mbox{$\sum_{i\in r}$} N_{i,T}\leq 8M^2 r,\]
i.e., with probability $\Omega(1)$, at time $T$ all but $8M^2r$ vertices have the same spin. In the next step, with probability $\geq (1/2)^{8M^2r+M^2}=\Omega(1)$, all these vertices plus the at most $M^2$ new components that get created (cf. \eqref{eq:Cstarbound}) get  the same spin as the component $C^*_{i,T+1}$, i.e., with probability $\Omega(1)$, at time $T+1$, all vertices have the same spin in $X_{T+1}$. 

Now consider two independent copies $\{X_t:t=0,1,\dots\}$ and $\{Y_t:t=0,1,\dots\}$ of the Swendsen-Wang chain. With probability $\Omega(1)$, we have that the spins in  $X_{T+1}$ and the spins in $Y_{T+1}$ are same (though the common spin value might be different in the two copies), and hence, conditioned on this occuring, we can couple them so that in the next step it holds that $X_{T+2}=Y_{T+2}$. Thus, by considering time intervals of length $T+2$, we obtain that there exists a constant $c>0$ such that for $T'=\left\lceil  c T\right\rceil=O(\log n)$, it holds that 
$\Pr(X_{T'}\ne Y_{T'})\le 1/4$.
This completes the proof of Theorem \ref{thm:main2} under condition~\ref{it:cona12}.

To prove Theorem \ref{thm:main2} under condition~\ref{it:conb12}, one only needs to establish the analogue of Lemma~\ref{lem:connected2} for the bipartite Erd\H{o}s-R\`enyi random graph; the rest of the argument is then completely analogous to the argument used for condition~\ref{it:cona12}. The following lemma  whose proof is given in Appendix~\ref{pflem:connected4} establishes the required cut properties, thus completing the proof of Theorem \ref{thm:main2}.
\begin{lemma}\label{lem:connected4}
Let $p\in (0,1)$ and $k\in(0,1]$ be arbitrary constants. Then, for every constant $M\geq 100/(kp)$,  the following holds with probability $1-e^{-\Omega(n)}$ over the choice of the graph $G=(V_L,V_R,E)\sim G(n,kn,p)$.  

Let $U_L\subseteq V_L, U_R\subseteq V_R$ be arbitrary subsets of vertices of $G$ with $|U_L|\geq n/10, |U_R|\geq k n/10$. Then, for every $S_L\subseteq U_L$, $S_R\subseteq U_R$ such that $|S_L|,|U_L\backslash S_L|\geq M^2$ and $|S_R|,|U_R\backslash S_R|\geq M^2$ , it holds that $\mathrm{cut}_{G[U_L\cup U_R]}(S_L\cup S_R)\geq M n$.
\end{lemma}

\subsection{Proof of 
Theorem \ref{thm:main}}\label{sec:pfthm:main}

In this section, we present the proof  of 
Theorem \ref{thm:main}. We provide the proof outlines
for the cases $B>2$ and $B<2$, and the proofs of the key lemmas are given 
in the appendix. 
We first define
$$(\alpha_L^*,\alpha_R^*):= \lim_{n\rightarrow\infty}\arg\max_{(\alpha_L,\alpha_R)}\Pr(\alpha_L,\alpha_R),$$
such $(\alpha_L^*,\alpha_R^*)$ uniquely exists as we state and prove in Lemma \ref{lem:maxprob} 
in Appendix \ref{sec:simplesw}.

\paragraph{Rapid mixing proof for $B>2$.} 
In this case, we will show first that, for any starting state,   the Swendsen-Wang chain moves in $O(1)$ iterations within constant distance from $(\alpha_L^*,\alpha_R^*)$ with probability $\Theta(1)$.
Then, we will show that the Swendsen-Wang chain moves within $O(n^{-1/2})$ distance from $(\alpha_L^*,\alpha_R^*)$ in $O(\log n)$ iterations with probability $\Theta(1)$.
Finally, using this fact, we will bound the mixing time via the coupling technique.
More formally, we introduce the following key lemmas. 
\begin{lemma}\label{lem:p1close1}
Let $\{X_t:t=0,1,\dots\}$ be the Swendsen-Wang chain on
a complete bipartite graph of size $(n,kn)$ with
any constants $k\geq 1, B>2$ and any starting state $X_0$.
For any constant $\delta>0$, there exists $T=O(1)$ such that $\|\alpha(X_T)-(\alpha_L^*,\alpha_R^*)\|_\infty\le\delta$ with probability $\Theta(1)$.
\end{lemma}

\begin{lemma}\label{lem:p1close2}
Let $\{X_t:t=0,1,\dots\}$ be the Swendsen-Wang chain on
a complete bipartite graph of size $(n,kn)$ with
any constants $k\geq 1, B>2$. 
There exist constants $\delta,L>0$ such that the following statement holds.
Suppose that we start at state $X_0$ such that $\|\alpha(X_0)-(\alpha_L^*,\alpha_R^*)\|_\infty\le\delta$.
Then, in $T=O(\log n)$ iterations,
the Swendsen-Wang chain moves to $X_{T}$ such that $\|\alpha(X_T)-(\alpha_L^*,\alpha_R^*)\|_\infty\le Ln^{-1/2}$ with probability $\Theta(1)$.
\end{lemma}

\begin{lemma}\label{lem:p1coupling1}
Let $\{X_t:t=0,1,\dots\}$, $\{Y_t:t=0,1,\dots\}$ be Swendsen-Wang chains on
a complete bipartite graph of size $(n,kn)$ with
any positive constants $k\geq 1, B\ne 2$.
Let $X_0,Y_0$ be a pair of configurations satisfying 
$$\|\alpha(X_0)-(\alpha_L^*,\alpha_R^*)\|_\infty,\|\alpha(Y_0)-(\alpha_L^*,\alpha_R^*)\|_\infty\le Ln^{-1/2}$$
for some constant $L>0$.
Then, there exists a coupling for $(X_t,Y_t)$ such that $\alpha(X_1)=\alpha(Y_1)$ with probability $\Theta(1)$.
\end{lemma}

\begin{lemma}\label{lem:p1coupling2}
Let $\{X_t:t=0,1,\dots\}$, $\{Y_t:t=0,1,\dots\}$ be Swendsen-Wang chains on
a complete bipartite graph of size $(n,kn)$ with
any constants $k\geq 1, B>0$.
For any constant $\varepsilon>0$,
there exist $T=O(\log n)$ and a coupling for $(X_t,Y_t)$ such that 
$\Pr[X_T\ne Y_T\,|\,\alpha(X_0)=\alpha(Y_0)]\le\varepsilon$.
\end{lemma}
The proofs of the above lemmas are presented in Appendices \ref{sec:p1close1}---\ref{sec:p1coupling1}.
Since the proof of Lemma \ref{lem:p1coupling2} is identical to that of \cite[Lemma 9]{galanis2015swendsen},
we omit it.
Now, we are ready to complete 
the proof  of 
Theorem \ref{thm:main} for $B>2$.

Consider two copies $X_t, Y_t$ under the Swendsen-Wang chain. We will show that for some $T=O(\log n)$, there exists a coupling such that
$\Pr[X_T\ne Y_T]\le1/4$. 
Let $\delta, L$ be as in Lemma \ref{lem:p1close1} and Lemma \ref{lem:p1close2}. 
Then, for some $T_1=O(1)$ with probability $\Theta(1)$, we have that
$$\|\alpha(X_{T_1})-(\alpha_L^*,\alpha_R^*)\|_\infty,\|\alpha(Y_{T_1})-(\alpha_L^*,\alpha_R^*)\|_\infty\le\delta.$$
Furthermore, for some $T_2=O(\log n)$ with probability $\Theta(1)$, we have that
\begin{equation}\label{eq:close}
\begin{split}
&\|\alpha(X_{T_1+T_2})-(\alpha_L^*,\alpha_R^*)\|_\infty\le Ln^{-1/2}\\
&\|\alpha(Y_{T_1+T_2})-(\alpha_L^*,\alpha_R^*)\|_\infty\le Ln^{-1/2}.
\end{split}
\end{equation}
Conditioning on \eqref{eq:close} and using Lemma \ref{lem:p1coupling1}, there exists a coupling that
$\alpha(X_{T_1+T_2+1})=\alpha(Y_{T_1+T_2+1})$ holds with probability $\Theta(1)$.
Conditioning on $\alpha(X_{T_1+T_2+1})=\alpha(Y_{T_1+T_2+1})$ and using Lemma \ref{lem:p1coupling2}, for any constant $\varepsilon^\prime>0$,
there exists $T_3=O(\log n)$ and another coupling such that $\Pr(X_{T_1+T_2+T_3+1}\ne Y_{T_1+T_2+T_3+1})\le\varepsilon^\prime$.
Since all events so far occur with probability $\Theta(1)$, there exists small enough constant $\varepsilon^\prime$ so that
$\Pr(X_T\ne Y_T)\le1/4$ for some $T=O(\log n)$ under some coupling.
This completes the proof of Theorem \ref{thm:main} for the case $B>2$.

\paragraph{Rapid mixing proof for $B<2$.}

In this case, we will show that $\alpha(X_t)$
moves within $O(n^{-1/2})$ distance from $(\alpha_L^*,\alpha_R^*)$ in $O(1)$ iterations.
Then, we will bound the mixing time via the coupling technique as before.
More formally, we introduce the following key lemmas.

\begin{lemma}\label{lem:p2close}
Let $\{X_t:t=0,1,\dots\}$ be the Swendsen-Wang chain on
a complete bipartite graph of size $(n,kn)$ with
any constants $k\geq 1, B<2$.
There exists a constant $L$ such that for any starting state $X_0$ after $T=O(1)$ iterations, the Swendsen-Wang chain moves to state $X_T$
such that $\|\alpha(X_T)-(\alpha_L^*,\alpha_R^*)\|_\infty\le L n^{-1/2}$ with probability $\Theta(1)$.
\end{lemma}

The proof of Lemma \ref{lem:p2close} is presented in Appendix~\ref{sec:p2close}.
By combining Lemmas \ref{lem:p1coupling1}-\ref{lem:p2close} and using the same arguments used for the case $B>2$, 
one can complete 
the proof  of 
Theorem \ref{thm:main} for
$B<2$.

\section{Experiments}\label{sec:exp}

\begin{figure*}[t!]
    \centering
    \begin{subfigure}[b]{0.49\textwidth}
        \includegraphics[width=\textwidth]{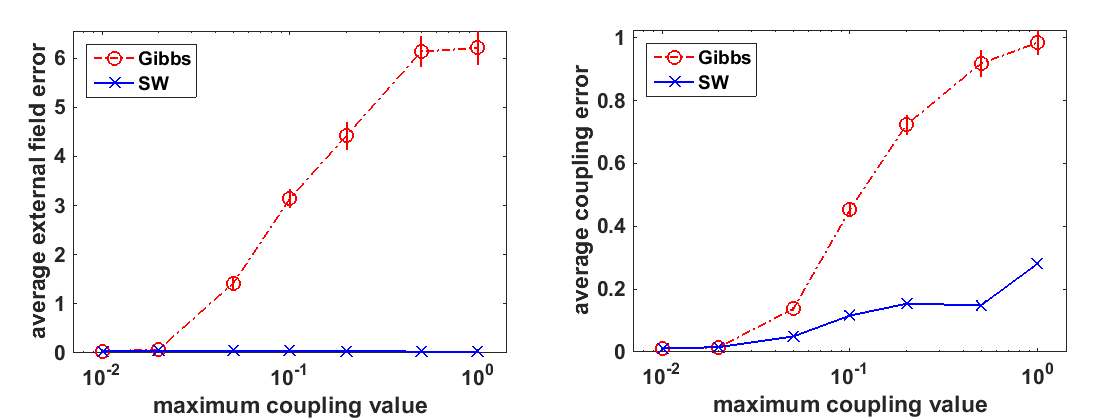}
        \caption{Facebook : $\gamma_v\sim\text{Unif}(0,0.1)$}
        \label{fig:fbp}
    \end{subfigure}
    \begin{subfigure}[b]{0.49\textwidth}
        \includegraphics[width=\textwidth]{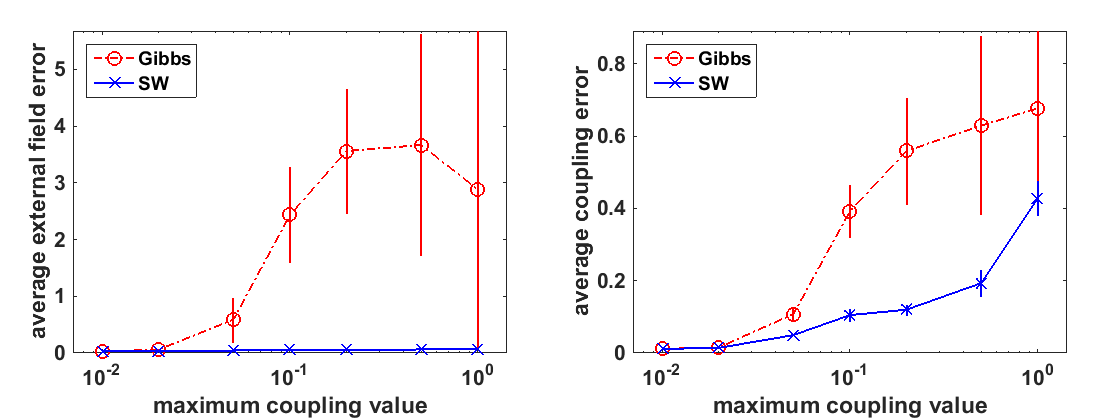}
        \caption{Facebook : $\gamma_v\sim\text{Unif}(-0.1,0.1)$}
        \label{fig:fbm}
    \end{subfigure}
    \begin{subfigure}[b]{0.49\textwidth}
        \includegraphics[width=\textwidth]{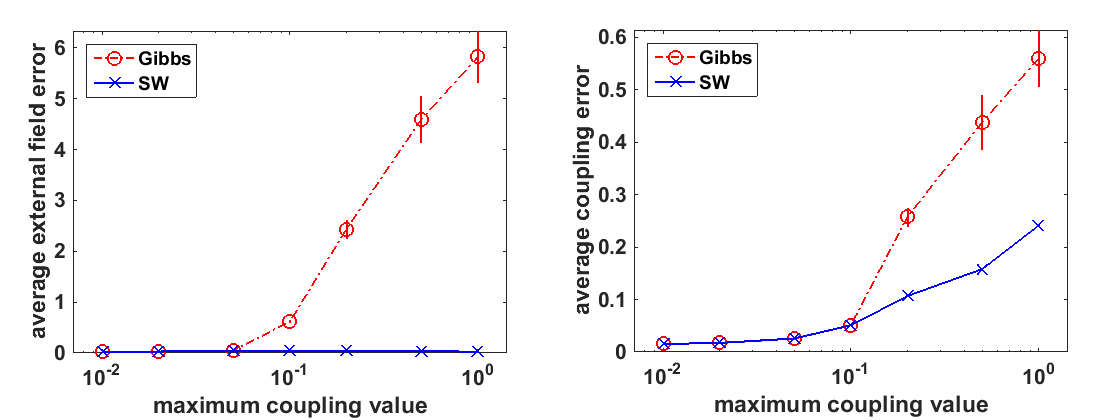}
        \caption{UCI : $\gamma_v\sim\text{Unif}(0,0.1)$}
        \label{fig:UCIp}
    \end{subfigure}
    \begin{subfigure}[b]{0.49\textwidth}
        \includegraphics[width=\textwidth]{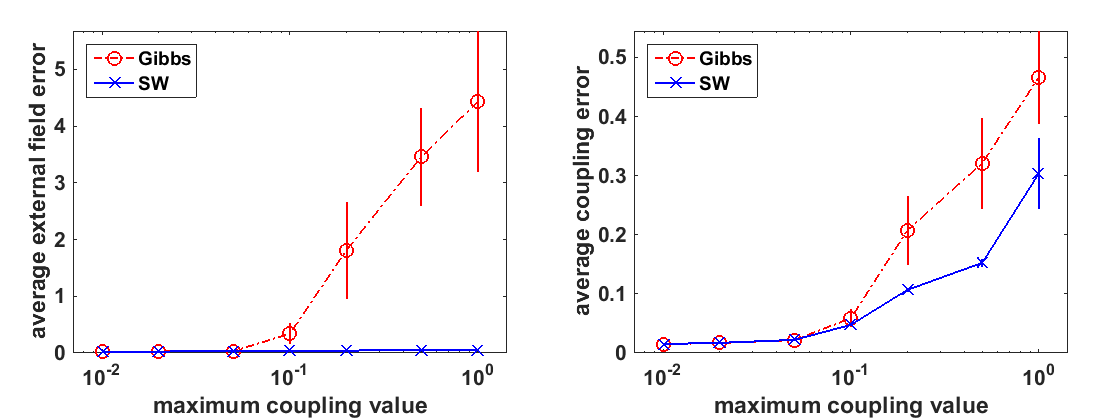}
        \caption{UCI : $\gamma_v\sim\text{Unif}(-0.1,0.1)$}
        \label{fig:UCIm}
    \end{subfigure}
    \begin{subfigure}[b]{0.49\textwidth}
        \includegraphics[width=\textwidth]{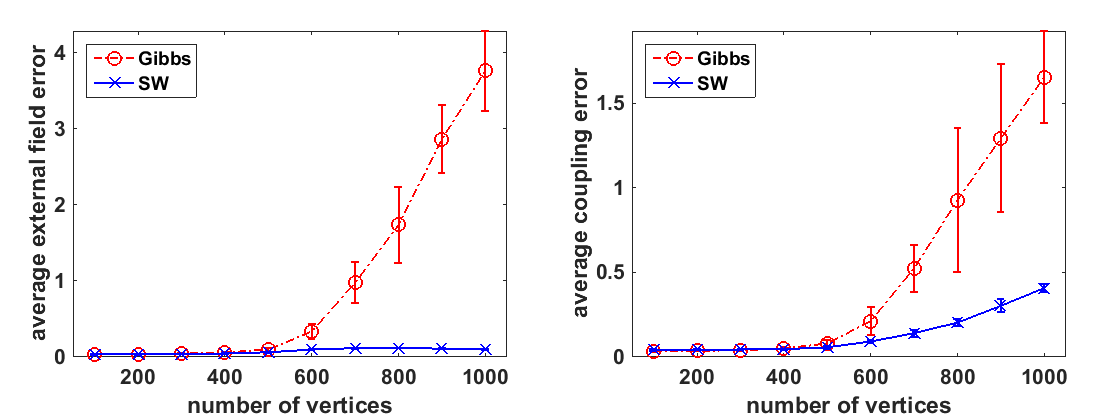}
        \caption{Synthetic : $\gamma_v\sim\text{Unif}(0,0.1)$}
        \label{fig:synp}
    \end{subfigure}
    \begin{subfigure}[b]{0.49\textwidth}
        \includegraphics[width=\textwidth]{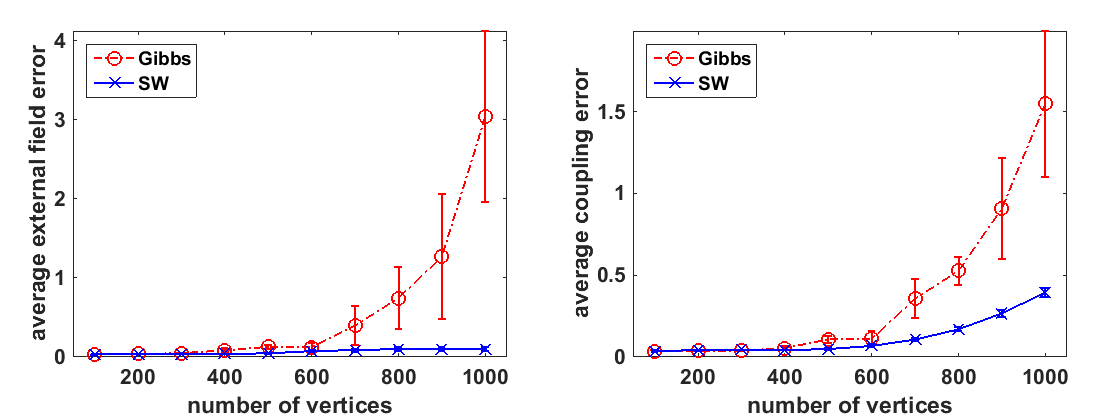}
        \caption{Synthetic : $\gamma_v\sim\text{Unif}(-0.1,0.1)$}
        \label{fig:synm}
    \end{subfigure}
    \caption{x-axis value  $x$ of (a), (b), (d), (e) is a range that $\beta_{uv}$ is sampled from, i.e. $\beta_{uv}\sim\text{Unif}(0,x)$, and x-axis value of (c), (f) is a number of vertices in a graph. 
    y-axis of external field error is a normalized external field error $\sum_{v\in V}|\gamma_v-\hat\gamma_v|/|V|$ and y-axis of coupling error is a normalized coupling error $\sum_{(u,v)\in E}|\beta_{uv}-\hat\beta_{uv}|/|E|$. 
    Each point is an average of 10 independent Ising models while each Ising model is learned by 1000 data samples. }
\vspace{-0.1in}
\end{figure*}

In this section, we compare the empirical performances of the Swendsen-Wang and the Gibbs chains 
for learning parameters of ferromagnetic Ising models.
We construct models on real world social graphs and synthetic stochastic partitioned graphs by assigning random parameters $[\beta_{uv}],[\gamma_v]$ on graphs.
For the choice of learning algorithm, we use the popular contrastive divergence (CD) algorithm \cite{hinton2012practical} which uses a Markov chain as its subroutine.

\paragraph{Data sets.} 
For each model, we generate a data set of 1000 samples  by running the Swendsen-Wang chain.
To construct a model,
we use two real world social graphs which are known to have certain partitioned structures, e.g.,
see \cite{girvan2002community}.
The first social graph is a Facebook graph consisting of 4039 nodes and 88234 edges, originally used in \cite{mcauley2012learning}.
Each node of the graph corresponds to an account of Facebook and each edge of the graph corresponds to a `friendship' in Facebook.
The second social graph is a UCI graph created from an online community consisting of 1899 nodes and 13838 edges, originally used in \cite{opsahl2009clustering}.
Each node in the graph corresponds to a student at the University of California, Irvine and each edge in the graph corresponds to the message log from April to October 2004, i.e. edge $(u,v)$ exists if $u$ sent message to $v$ or vice versa.
For the real world social graphs, we assign $\gamma_v\sim\text{Unif}(0,0.1), \text{Unif}(-0.1,0.1)$, i.e., both positive and mixed external field, and $\beta_{uv}\sim\text{Unif}(0,x)$ where $x\in [0.01, 1]$.\footnote{$\text{Unif}(a,b)$
denotes the random variable chosen in the interval $[a,b]$ uniformly at random.}
For given $x$, we sample 10 i.i.d.\ $[\beta_{uv}]$ to obtain 10 different models.

Our synthetic stochastic partitioned graphs are bipartite random graphs of 100 to 1000 vertices with two partitions of same size, i.e. $|V_1|=|V_2|$.
We set the inter-partition edge probability $p_{11}=p_{22}=0.007$ and the intra-partition edge probability $p_{12}=0.003$.
For each graph size, we sample 10 bipartite random graphs.
For synthetic graphs, we assign $\gamma_v\sim\text{Unif}(0,0.1),\text{Unif}(-0.1,0.1)$ and $\beta_{uv}\sim\text{Unif}(0,1)$.

\vspace{-0.07in}
\paragraph{Contrastive divergence learning.}
Given a data set, the most standard way to estimate/recover
parameters of a `hidden' model is 
the log-likelihood maximization. To this end, 
it is known \cite{wainwright2008graphical} that computing the gradients of a graphical model requires the
computation of 
marginal probabilities, e.g., $E[\sigma_u\sigma_v]$ and $E[\sigma_v]$, and one can run a Markov chain to estimate them.
However, this is not efficient since the Markov chain has to be run for large enough iterations
until it mixes.
To address the issue, 
the contrastive divergence (CD) learning algorithm \cite{hinton2012practical}
suggests that it suffices to run a Markov chain 
for a fixed number of iterations to approximate each gradient.
The underlying intuition under CD learning is that 
it is not necessary to wait for mixing
for each gradient update
since the parameters are  changing slowly and mixing effects are amortized over iterations.
The detailed procedure of the algorithm is presented in Algorithm \ref{alg:cd}. 

\begin{algorithm}[h]
\caption{Contrastive Divergence Learning} \label{alg:cd}
\begin{algorithmic}[1]
\State \textbf{Input:} $n_i$, $\eta(\cdot)$, $k$, $n_s$, $\texttt{MC}(\cdot,\cdot)$, $\mu_{uv}$, $\mu_v$
\State \textbf{Output:} Estimated parameters $[\hat\beta_{uv}]$, $[\hat\gamma_v]$
\State \textbf{Initialization:} $i,\hat\beta_{uv},\hat\gamma_v\leftarrow 0$ and randomly initialize states $\sigma^1,\dots,\sigma^{n_s}$ of Ising model
\While{$i<n_i$}
\State $s\leftarrow 0$
\While{$s<n_s$}
\State $\sigma^s\leftarrow\texttt{MC}(\sigma^s,k)$
\State $s\leftarrow s+1$
\EndWhile
\State $\hat\mu_{uv}\leftarrow\frac{1}{n_s}\sum_{s=1}^{n_s}\sigma^s_u\sigma^s_v$
\State $\hat\mu_{v}\leftarrow\frac{1}{n_s}\sum_{s=1}^{n_s}\sigma^s_v$
\State $\hat\beta_{uv}\leftarrow\hat\beta_{uv}+\eta(i)(\mu_{uv}-\hat\mu_{uv})$ for all $(u,v)\in E$
\State $\hat\gamma_v\leftarrow\hat\gamma_v+\eta(i)(\mu_{v}-\hat\mu_{v})$ for all $v\in V$
\State $i\leftarrow i+1$
\EndWhile
\end{algorithmic}
\end{algorithm}

In Algorithm \ref{alg:cd}, 
we denote by $\texttt{MC}(\sigma,k)$ a state of the Ising model generated from running $k$ iterations of the Markov chain $\texttt{MC}(\cdot,\cdot)$ starting from the state $\sigma$, and
$\mu_{uv}=E[\sigma_u\sigma_v], \mu_v=E[\sigma_v]$
are the empirical marginals from the data set.
In addition, $n_i$, $\eta(\cdot)$, $k$, $n_s$ denote
the number of gradient updates, the step size (or learning rate),
the number of samples and the number of MC updates, respectively,
which are hyper parameters of the CD algorithm.
Since the Swendsen-Wang chain takes $O(|V|)$ times longer per each iteration,
we use $k=1$ and $k=|V|$
for the Swendsen-Wang chain and the Gibbs chain,
respectively, for fair comparisons.

\paragraph{Experimental results.}
In our experiments, we
observe that the Swendsen-Wang chain outperforms the Gibbs chain, where the gap is significant as $\beta_{uv}$ or the graph size are large.
Our experimental results on real world graphs are reported in 
Figure \ref{fig:fbp}, \ref{fig:UCIp}, \ref{fig:fbm}, \ref{fig:UCIm}, which
show that the Swendsen-Wang chain outperforms the Gibbs chain for both errors on $[\gamma_v]$ and $[\beta_{uv}]$.
One can observe that the error difference between the Swendsen-Wang chain and the Gibbs chain grows as interaction strength $[\beta_{uv}]$ increases, which
is because the Gibbs chain mixes slower at low temperatures. 
Furthermore, the variance of errors of the Gibbs chain increases while the variance of the Swendsen-Wang chain remains small.
Our experimental results using synthetic graphs are similar to 
those of the real world social graphs.
Figures \ref{fig:synp}, \ref{fig:synm} show that the Swendsen-Wang chain also
outperforms the Gibbs chain as the graph size grows.
We observe that the external field error of the Gibbs chain
increases as the graph size increases while that of the Swendsen-Wang chain remains small.

\section{Conclusion}\label{conc}
Despite the rich expressive power of graphical models, the associated expensive inference tasks have
been the key bottleneck for their large-scale applications. In this paper, we prove that the
Swendsen-Wang sampler mixes fast for stochastic partitioned attractive GMs, where our mixing bound
$O(\log n)$ is quite practical for large-scale instances. We believe that our findings have further potential applications even for general (not necessarily, attractive) GMs if one can approximate 
a non-attractive model by an attractive one;
it was recently shown that
any binary pairwise GM can be approximated by an attractive binary pairwise GM
on the so-called $2$-cover graph having two partitions \cite{ruozzi2014making}.
For example, one can use the Swendsen-Wang sampler to learn parameters of 
the $2$-cover attractive model and further fine-tune them using the Gibbs sampler
on the original model.
This is an interesting future research direction.

\bibliographystyle{plain}
\bibliography{reference}

\appendix

\section{Proofs of Key Lemmas for Theorem \ref{thm:main2}}\label{pflem:main2}
\subsection{Proof of Lemma \ref{lem:connected2}}\label{pflem:connected2}
Let $G=(V,E)\sim G(n,p)$ and $U$ be an arbitrary subset of $V$ such that $|U|\geq n/10$. Further, let $S\subset U$ be a set such that $|S|,|U\backslash S|\geq M^2$, where recall that $M$ is a constant satisfying $M\geq 100/p$. Note that $\mathrm{cut}_{G[U]}(S)$ is just the number of edges between the sets $S$ and $U\backslash S$ and thus 
\[E[\mathrm{cut}_{G[U]}(S)]=p\cdot |S|\cdot|U\backslash S|\geq M^2(n/10-M^2)p\geq 9M n.\]
Thus, by the Chernoff bound, we obtain that the probability that $\mathrm{cut}_{G[U]}(S)<Mn$ is at most $e^{-Mn}\leq e^{-10n}$. There are at most $2^n$ ways to choose the set $U$ and at most $2^n$ ways to choose the set $S\subseteq U$. Thus, the lemma follows by taking a union bound over all possible choices of the sets $U,S$. 

This completes the proof of the lemma.

\subsection{Proof of Lemma \ref{lem:connected4}}\label{pflem:connected4}
Let $G=(V_L,V_R,E)\sim G(n,kn,p)$, and $U_L\subseteq V_L, U_R\subseteq V_R$ be arbitrary subsets of vertices with $|U_L|\geq n/10, |U_R|\geq k n/10$. Further, let $S_L\subseteq U_L$, $S_R\subseteq U_R$ be subsets such that $|S_L|,|U_L\backslash S_L|\geq M^2$ and $|S_R|,|U_R\backslash S_R|\geq M^2$, where recall that $M$ is a constant satisfying $M\geq 100/(pk)$. For convenience, set $U:=U_L\cup U_R$ and $S:=S_L\cup S_R$. We are interested in $\mathrm{cut}_{G[U]}(S)$ which is  the number of edges between the sets $S$ and $U\backslash S$. Thus,
\begin{align*}
E[\mathrm{cut}_{G[U]}(S)]&=p(|S_L|\cdot|U_R\backslash S_R|+|S_R|\cdot|U_L\backslash S_L|)\geq 2p\big(|S_L|\cdot|U_R\backslash S_R|\cdot|S_R|\cdot|U_L\backslash S_L|\big)^{1/2}\\
&\geq 2p\big( M^4(kn/10-M^2)(n/10-M^2)\big)^{1/2}\geq 9M n.
\end{align*}
Thus, by the Chernoff bound, we obtain that the probability that $\mathrm{cut}_{G[U]}(S)<Mn$ is at most $e^{-Mn}\leq e^{-10n}$. Since $k\in (0,1]$, there are at most $2^{n(k+1)}\leq 2^{2n}$ ways to choose the sets $U_L,U_R$ and at most $2^{n(k+1)}\leq2^{2n}$ ways to choose the sets $S_L,S_R$. Thus, the lemma follows by taking a union bound over all possible choices of the sets $U_L,U_R,S_L,S_R$.

This completes the proof of the lemma.

\section{Proofs of Key Lemmas for Theorem \ref{thm:main}}\label{sec:rapidpf}
In this section, we provide the proofs of Lemmas \ref{lem:p1close1}-\ref{lem:p2close}.
To this end, we first introduce a two-dimensional 
function $F$ which captures the behaviour of the Swendsen-Wang dynamics and introduce the connection between $F$ and the Ising model.
Throughout this section, we only consider the Ising model on the complete bipartite graph of size $(n,kn)$ with 
$$\beta_{uv}=-\frac{1}{2}\log\left(1-\frac{B}{n\sqrt{k}}\right),~\gamma_v=0~\quad\text{for all }(u,v)\in E,~v\in V,$$
where $B>0$ is some constant.

\subsection{Simplified Swendsen-Wang}\label{sec:simplesw}
We first introduce the following result \cite{johansson2012giant} about the giant component of the bipartite Erd\H{o}s-R\'enyi random graph.
\begin{lemma}[\protect{\cite[Theorem 6, Theorem  12]{johansson2012giant}}]\label{thm:giant}
Consider the bipartite Erd\H{o}s-R\'enyi random graph $$G=(V_L,V_R,E)=G(n,kn,p)$$
where 
$p=\frac{B}{n\sqrt{k}}$ for some constant $B>0$ and $k\geq1$ is some constant.
Then, the following statements hold a.a.s.
\begin{itemize}
\item[(a)] For $B<1$, the largest (connected) component of $G$ has size $O(\log n)$.
\item[(b)] For $B>1$, the following event happens: $G$ has a unique ``giant'' component which consists of $\theta_R kn(1+o(1))$ vertices in $V_R$ and $\theta_L n(1+o(1))$ vertices in $V_L$
where $\theta_R$ is the unique positive solution of
\begin{equation}\label{eq:theta1}
\theta_R+\exp\left(\frac{B}{\sqrt{k}}\left(\exp\left(-B\sqrt{k}\theta_R\right)-1\right)\right)=1
\end{equation}
and $\theta_L$ is the unique positive solution of
\begin{equation}\label{eq:theta2}
\theta_L+\exp\left(B\sqrt{k}\left(\exp\left(-\frac{B\theta_L}{\sqrt{k}}\right)-1\right)\right)=1.
\end{equation}
The second largest component of $G$ has size $O(\log^2 n)$.
\item[(c)] For $B=1$, the largest component of $G$ has size $o(n)$.
\end{itemize}
\end{lemma}
By simple calculations, one can observe that \eqref{eq:theta1}, \eqref{eq:theta2} reduce to 
\begin{equation}\label{eq:theta4}
\exp(-B\sqrt{k}\theta_R)=1-\theta_L\qquad\qquad\exp\left(-\frac{B}{\sqrt{k}}\theta_L\right)=1-\theta_R.
\end{equation}
Now, consider the Ising model on the complete bipartite graph $G=(V_L,V_R,E)$ of size $(n,kn)$.
We briefly explain what happens in a single iteration of the Swendsen-Wang chain on $G$ for each step asymptotically.
Given a spin configuration $\sigma$ with $\alpha(\sigma)=(\alpha_L,\alpha_R)$, 
the step 2 of the Swendsen-Wang dynamics starting from $\sigma$ is equivalent to sampling two bipartite Erd\H{o}s-R\'enyi random graphs $G(\alpha_Ln,\alpha_Rkn,p)$, $G((1-\alpha_L)n,(1-\alpha_R)kn,p)$ where $p=\frac{B}{n\sqrt{k}}$.

Suppose $(1-\alpha_L)(1-\alpha_R)B\le1$ and $\alpha_L\alpha_RB>1$. 
Then, by Lemma \ref{thm:giant}, there exists a single giant component of size $(\theta_L\alpha_Ln,\theta_R\alpha_Rkn)$ where $(\theta_L,\theta_R)$ is a unique positive solution of 
\begin{equation}\label{eq:thetaF}
    \exp(-B\sqrt{k}\alpha_R\theta_R)=1-\theta_L\qquad\qquad
    \exp\left(-\frac{B}{\sqrt{k}}\alpha_L\theta_L\right)=1-\theta_R,
\end{equation}
and the other `small' components have size $o(n)$ a.a.s.\ after step 2 of the Swendsen-Wang dynamics.
One can notice that \eqref{eq:thetaF} is equivalent to \eqref{eq:theta4} by substituting $n\leftarrow \alpha_Ln$, $k\leftarrow\frac{k\alpha_R}{\alpha_L}$ and $B\leftarrow \sqrt{\alpha_L\alpha_R}B$.
At step 3 of the Swendsen-Wang dynamics, asymptotically a half of the small components, which have size $\left((1-\theta_L\alpha_L)n/2,(1-\theta_R\alpha_R)kn/2\right)$, receive same spin with the giant component.
Now suppose $(1-\alpha_L)(1-\alpha_R)B,\alpha_L\alpha_RB\le1$.
Then after the step 2 of the Swendsen-Wang dynamics, every connected component has size $O(\log n)$. 
After step 3 of the Swendsen-Wang dynamics, as each spin class asymptotically have a half of the vertices of $V_L,V_R$, 
it 
outputs a phase $(1/2,1/2)$ asymptotically.
We ignore the case $(1-\alpha_L)(1-\alpha_R)B>1$ for now, i.e. we ignore the giant component of the smaller spin class,
which will be handled in the proof of Lemma \ref{lem:p1close1}.
Under these intuitions,
one can expect that
the following function $F$ captures the behavior of the Swendsen-Wang chain (ignoring the giant component of the smaller spin class) on the complete bipartite graph.
\begin{equation}\label{eq:simplifiedSW}
F(\alpha_L,\alpha_R):=(F_L,F_R)=\left(\frac{1}{2}\left(1+\theta_L \alpha_L\right),\frac{1}{2}\left(1+\theta_R \alpha_R\right)\right)
\end{equation}
where $$(\theta_L,\theta_R)=
\begin{cases}
\qquad\qquad (0,0)& \mbox{for}~ \sqrt{\alpha_L\alpha_R} B\le 1\\
\mbox{the unique solution of \eqref{eq:thetaF}}
&\mbox{for}~\sqrt{\alpha_L\alpha_R} B> 1
\end{cases}.$$
We note that $F$ is continuous on $[0,1]^2$.
Formally, one can prove 
the following lemma about the relation between the function $F$ and the Swendsen-Wang chain; we omit its proof since it is elementary under the above intuitions.
\begin{lemma}\label{lem:SWF}
Let $\{X_t:t=0,1,\dots\}$ be the Swendsen-Wang chain on
a complete bipartite graph of size $(n,kn)$ with
any constants $B\ne2$ and starting phase $\alpha(X_0)=(\alpha_L,\alpha_R)$.
If $\alpha_L\alpha_RB\ne 1$ and $(1-\alpha_L)(1-\alpha_R)B\le1$, i.e., the smaller spin class is subcritical, then
$\alpha(X_1)=F(\alpha_L,\alpha_R)+(o(1),o(1))$ a.a.s.
\end{lemma}

From the definition of $F$, $(\alpha_L,\alpha_R)$ is a fixed point of $F$ if and only if $\alpha_L=\frac{1}{2}+\frac{1}{2}\theta_L \alpha_L$, $\alpha_R=\frac{1}{2}+\frac{1}{2}\theta_R \alpha_R$,
i.e., $\theta_L=\frac{2\alpha_L-1}{\alpha_L},~\theta_R=\frac{2\alpha_R-1}{\alpha_R}$.
Substituting this relation into \eqref{eq:thetaF} yields that every fixed point of $F$ must satisfy the following
equations
\begin{equation}\label{eq:fixedF}
\exp\left(B\sqrt{k}(1-2\alpha_R)\right)=\frac{1-\alpha_L}{\alpha_L}\qquad\qquad\exp\left(\frac{B}{\sqrt{k}}(1-2\alpha_L)\right)=\frac{1-\alpha_R}{\alpha_R}.
\end{equation}
One expects that the Swendsen-Wang chain, starting from a phase which corresponds to a fixed point of $F$, will stay around the fixed point.
Now we introduce two lemmas about the fixed points of $F$. Lemma \ref{lem:attractive} shows that $F$ has a unique fixed point which is Jacobian attractive.
Further, Lemma \ref{lem:convergence} guarantees that for any starting point $(\alpha_L,\alpha_R)$, 
$$F^{(t)}(\alpha_L,\alpha_R):=\underbrace{F\circ\dots\circ F}_t(\alpha_L,\alpha_R)$$
converges to the fixed point of $F$ as $t\rightarrow\infty$.
\begin{lemma}\label{lem:attractive}
The following hold: 
\begin{itemize}
\item[1.]For constant $B<2$, $(1/2,1/2)$ is the unique fixed point of $F$ and it is Jacobian attractive.
\item[2.]For constant $B>2$, the solution $\alpha_L^*,\alpha_R^*\in(1/2,1]$ of \eqref{eq:fixedF} is the unique fixed point of $F$ and it is Jacobian attractive.
\end{itemize}
\end{lemma}

\begin{lemma}\label{lem:convergence}
For any point $(\alpha_L,\alpha_R)\in[0,1]^2$,  $F^{(t)}(\alpha_L,\alpha_R)$ converges to
the unique fixed point of $F$ as $t\rightarrow\infty$.
\end{lemma}
The proofs of the above lemmas are presented in Sections \ref{sec:pflem:attractive} and \ref{sec:pflem:convergence}, respectively.

Finally, we provide the connection between $F$ and the Ising model.
Suppose the probability of some phase, say $(\alpha_L^\prime,\alpha_R^\prime)$,  of the Ising model on the complete bipartite graph of size $(n,kn)$ dominates that of other phases,
i.e., $\mu\big((\alpha_L^\prime,\alpha_R^\prime)\pm(\Theta(1),\Theta(1))\big)=1-o(1)$.
Then the Swendsen-Wang chain must converge to $(\alpha_L^\prime,\alpha_R^\prime)$ a.a.s.
Since $F$ converges to its unique fixed point by Lemma \ref{lem:convergence}, one can naturally expect that the fixed point of $F$ is equivalent to $(\alpha_L^\prime,\alpha_R^\prime)$.
The following lemma establishes this intuition formally.
\begin{lemma}\label{lem:maxprob}
For the Ising model on the complete bipartite graph of size $(n,kn)$ with 
$\beta_{uv}=-\frac{1}{2}\log\left({1-\frac{B}{n\sqrt{k}}}\right)$
for some constant $B>0$ and $\gamma_v=0$,
the `maximum a posteriori phase' is
$$\lim_{n\rightarrow\infty}\arg\max_{(\alpha_L,\alpha_R)}\Pr(\alpha_L,\alpha_R)=\begin{cases}
\left(\frac{1}{2},\frac{1}{2}\right)\qquad\text{for}~B\le 2\\
(\alpha_L^*,\alpha_R^*)\quad\text{for}~B>2
\end{cases}$$
where $\alpha^*_L,\alpha^*_R\in(1/2,1]$ is the unique solution of \eqref{eq:fixedF}.
\end{lemma}
The proof of the above lemma is presented in Section \ref{sec:pflem:maxprob}.

\subsection{Proof of Lemma \ref{lem:p1close1}}\label{sec:p1close1}

In this section, we prove Lemma \ref{lem:p1close1}. 

Clearly, it suffices to show the lemma for all sufficiently small $\delta>0$. We start by establishing the following claim.
\begin{claim}\label{lem:smallsub}
For any constant $B>2$ and any fixed point $(\alpha_L^*,\alpha_R^*)$ of $F$, the following inequality holds
$$(1-\alpha_L^*)(1-\alpha_R^*)B^2<1,$$
i.e., the smaller spin class of the phase corresponding to the fixed point of $F$ is subcritical.
\end{claim}

\begin{proof}
Using the parametrization $z_L^*=2\alpha_L^*-1,~z_R^*=2\alpha_R^*-1$, we have
\begin{equation}\label{eq:smallsub}
   (1-\alpha_L^*)(1-\alpha_R^*)B^2=\frac{1}{4}\frac{(1-z_L^*)(1-z_R^*)}{z_L^*z_R^*}
\log\frac{1+z_L^*}{1-z_L^*}\log\frac{1+z_R^*}{1-z_R^*},
\end{equation}
where we used the fact that $(\alpha_L^*,\alpha_R^*)$ satisfies \eqref{eq:fixedF}.
In the proof of Lemma \ref{lem:maxprob}, we show that
\eqref{eq:nsd4} holds.

This completes the proof of Claim \ref{lem:smallsub}.
\end{proof}
Due to Claim~\ref{lem:smallsub}, 
for all sufficiently small  $\delta>0$, we have that
$(1-\alpha_L^*+\delta)(1-\alpha_R^*+\delta)B^2<1$.
Now, for $B>2$, Lemma \ref{lem:convergence} implies that there exists a constant $T_1$ such that 
$$F^{(T_1)}([0,1]^2)\subset [\alpha_L^*-\delta,\alpha_L^*+\delta]\times
[\alpha_R^*-\delta,\alpha_R^*+\delta].$$
First, suppose $F(1-\alpha_{L,0},1-\alpha_{R,0})=(1/2,1/2)$, i.e. the smaller spin class is subcritical.
Then, in $T_1$ iterations, the Swendsen-Wang chain moves $l_\infty$-distance $\delta$ from $(\alpha_L^*,\alpha_R^*)$ with probability $1-o(1)$ 
due to Lemma \ref{lem:SWF}.
Now, consider the case $F(1-\alpha_{L,0},1-\alpha_{R,0})>(1/2,1/2)$, i.e. two giant components appears in both spins in the step 2 of the Swendsen-Wang dynamics.
Then, giant components merge with probability $1/2$ and
it results $\alpha(X_{T_1})>(\alpha_L^*-\delta,\alpha_R^*-\delta)$ with probability $\Theta(1)$.
Therefore, starting from $\alpha(X_{T_1})>(\alpha_L^*-\delta,\alpha_R^*-\delta)$, the Swendsen-Wang chain also moves within $l_\infty$-distance $\delta$ from $(\alpha_L^*,\alpha_R^*)$ 
in $T_1$ iterations with probability $1-o(1)$ due to Lemma \ref{lem:SWF}.
This completes the proof of Lemma \ref{lem:p1close1}.

\subsection{Proof of Lemma \ref{lem:p1close2}}\label{sec:p1close2}
In this section, we prove Lemma \ref{lem:p1close2}.

By Lemma \ref{lem:attractive}, we have that $(\alpha_L^*,\alpha_R^*)$ is a Jacobian attractive fixed point of $F$. Using the bound in Claim \ref{lem:smallsub}, we thus obtain that 
there exist constants $\delta>0, c<1$ 
such that $(1-\alpha_L^*+\delta)(1-\alpha_R^*+\delta)B^2<1$ and 
$$|F(\alpha_L,\alpha_R)-(\alpha^*_L,\alpha^*_R)|\le c|(\alpha_L,\alpha_R)-(\alpha^*_L,\alpha^*_R)|,$$for all $\alpha_L\in[\alpha_L^*-\delta,\alpha_L^*+\delta],~\alpha_R\in[\alpha_R^*-\delta,\alpha_R^*+\delta]$.
For the proof of Lemma \ref{lem:p1close2},
we assume that for some $t$, the event $\|\alpha(X_t)-(\alpha^*_L,\alpha^*_R)\|_\infty\le\delta$ occurs
(initially at $t=0$, it occurs)
and introduce the following two lemmas.

\begin{lemma}\label{lem:concentration1}
Consider the bipartite Erd\H{o}s-R\'enyi random graph $G(n,kn,p)$ where $p=\frac{B}{n\sqrt{k}}$ 
for some constants $B>0$ and $k\geq 1$.
Let $C_1,C_2,\dots$ be the connected components of $G$
in decreasing order of size.
Then, there exist constants $K_1,K_2>0$ such that 
\begin{enumerate}[label=(\alph*),ref=(\alph*),leftmargin=*]
    \item \label{it:wsxb12a} for $B<1$, we have
$$E\bigg[\sum_{i\ge 1}|C_i|^2\bigg]\le K_1 n,$$
    \item \label{it:wsxb12b} for $B>1$, we have
$$E\bigg[\sum_{i\ge 2}|C_i|^2\bigg]\le K_2 n,$$
\end{enumerate}

\end{lemma}
\begin{lemma}\label{lem:normal}
Consider the Swendsen-Wang dynamics
on the complete bipartite graph of size $(n,kn)$ with some constant $k\geq 1$,
$\beta_{uv}=-\frac{1}{2}\log\left({1-\frac{B}{n\sqrt{k}}}\right)$
for some constant $B>2$ and $\gamma_v=0$.
Let $C_1,C_2,\dots$ be the connected components of $G$
in decreasing order of size after the step 2 of the Swendsen-Wang dynamics.
Then, given the event $\sum_{i\ge 2}|C_i|^2< w Kn$ for some $w\geq 1$ and $K>0$,
it holds that
$$\Pr\left(\big||C_1\cap V_L|-\theta_{L}n\big|,\big||C_1\cap V_R|-\theta_{R}kn\big|\le w\sqrt{n}
\right)\ge1-\frac{2K}{w}-\frac{1+k}{w^2},$$
where $(\theta_{L},\theta_{R})$ is the unique positive solution of
\eqref{eq:thetaF}.
\end{lemma}
The proofs of Lemmas \ref{lem:concentration1} and \ref{lem:normal} are presented in Sections \ref{sec:pflem:concentration1} and \ref{sec:pflem:normal}, respectively.
From $(1-\alpha_L^*+\delta)(1-\alpha_R^*+\delta)B^2<1$, $\|\boldsymbol\alpha(X_t)-(\alpha^*_L,\alpha^*_R)\|_\infty\le\delta$ and
Lemma \ref{lem:concentration1}, after the step 2 of the Swendsen-Wang dynamics (starting from $X_t$), we have
$$E\bigg[\sum_{i\ge 2}|C_i|^2\bigg]\le K n$$
for some constant $K$.
Hence, by Markov's inequality,
for any $w_t\ge 1$, we have
\begin{equation}\label{eq:p1markov}
\Pr\bigg(\sum_{i\ge 2}|C_i|^2< w_t Kn\bigg)\ge 1-1/w_t.
\end{equation}
We will specify the value of $w_t$ later. For now, assume that the event $\sum_{i\ge 2}|C_i|^2< w_t Kn$ occurs.
Then, from Azuma's inequality, the number $Z_i$ of vertices
that receive spin $i$ in $V\setminus C_1$ in the step 3 of the Swendsen-Wang dynamics
is concentrated around its expectation as
\begin{align*}
&\Pr\left(\big|Z_i\cap V_L-E[Z_i\cap V_L]\big|\ge w_t\sqrt{Kn}\right)\le 2\exp(-w_t/2)\\
&\Pr\left(\big|Z_i\cap V_R-E[Z_i\cap V_R]\big|\ge w_t\sqrt{Kn}\right)\le 2\exp(-w_t/2).
\end{align*}
Using union bound, we obtain that
\begin{align}\label{eq:p1azuma}
\Pr\left(\big|Z_i\cap V_j-E[Z_i\cap V_j]\big|\ge w_t\sqrt{Kn}~~\text{for any}~i\in\{-1,1\},j\in\{L,R\}\right)\le 8\exp(-w_t/2).
\end{align}
On the other hand, using Lemma \ref{lem:normal}, we can bound the deviation of the size of the giant component as
\begin{equation}\label{eq:giantbound}
    \big||C_1\cap V_L|-\alpha_L(X_t)\theta_{L}n\big|,\ \big||C_1\cap V_R|-\alpha_R(X_t)\theta_{R}kn\big|\le w_t\sqrt{n}
\end{equation}
with probability at least
$$1-\frac{U_1}{w_t}-\frac{U_2}{w_t^2}$$
for some constants $U_1,U_2>0$, where
such $U_1,U_2$ exist as $\frac{1}{2k}\le\frac{\alpha_R(X_t)kn}{\alpha_L(X_t)n}\le 2k$.
By combining \eqref{eq:p1markov}, \eqref{eq:p1azuma} and \eqref{eq:giantbound}, we obtain
\begin{equation}\label{eq:p1distbound1}
\|\boldsymbol\alpha(X_{t+1})-F(\boldsymbol\alpha(X_t))\|_\infty\le w_t(1+\sqrt{K})n^{-1/2}
\end{equation}
with probability at least
$$(1-1/w_t)\left(1-8\exp\left(-\frac{w_t}2\right)-\frac{U_1}{w_t}-\frac{U_2}{w_t^2}\right).$$
Furthermore,
by combining \eqref{eq:p1distbound1} and $|F(\alpha_L,\alpha_R)-(\alpha^*_L,\alpha^*_R)|\le c|(\alpha_L,\alpha_R)-(\alpha^*_L,\alpha^*_R)|$, 
it follows that 
\begin{equation}\label{eq:p1distbound2}
\|\boldsymbol\alpha(X_{t+1})-(\alpha^*_L,\alpha^*_R)\|_\infty\le\frac{c+1}{2}\|\boldsymbol\alpha(X_t)-(\alpha^*_L,\alpha^*_R)\|_\infty\leq \delta
\end{equation}
by setting $w_t$ as
$$w_t:=\frac{1-c}{2}\frac{n^{1/2}}{1+\sqrt{K}}\|\boldsymbol\alpha(X_t)-(\alpha^*_L,\alpha^*_R)\|_\infty\ge
\frac{1-c}{2}\frac{L}{1+\sqrt{K}}.$$
Namely, $\|\boldsymbol\alpha(X_t)-(\alpha^*_L,\alpha^*_R)\|_\infty$ and $w_t$
decrease 
with at least multiplicative factor $(c+1)/2$.
Therefore, by applying the above arguments from $t=0, 1,\dots$, there exists $T=O(\log n)$ such that
$$\|\boldsymbol\alpha(X_T)-(\alpha^*_L,\alpha^*_R)\|_\infty\le Ln^{-1/2},$$
with probability at least
\begin{align*}
&\prod_{t= 0}^{T-1}
\left(1-\frac1{w_t}\right)\left(1-8\exp\left(-\frac{w_t}2\right)-\frac{U_1}{w_t}-\frac{U_2}{w_t^2}\right)\\
&\geq \prod_{t= 0}^{T-1}\exp\left(-\frac{2s}{w_t}\right)\\
&\geq \prod_{t= 0}^{\infty}
\exp\left(-\frac{2s}{w_t}\right)\\
&=\exp\left(-\frac{4s}{1-c}\frac{1+\sqrt{K}}{L}\sum_{t=0}^\infty \left(\frac{1+c}{2}\right)^t\right)\\
&=\Theta(1),
\end{align*}
where 
the first inequality is elementary to check by
defining $s:=\max(U_1,U_2+1,10)$ and assuming large enough $L$ so that
$w_t\ge\max(U_1^2,(U_2+1)^2,100)$, without loss of generality.

This completes the proof of Lemma \ref{lem:p1close2}.

\subsection{Proof of Lemma \ref{lem:p1coupling1}}\label{sec:p1coupling1}
In this proof, we prove Lemma \ref{lem:p1coupling1} for the case $B>2$.
One can apply the same argument for the case $B<2$.
Let $\{V_L,V_R\}$, $|V_L|=n,|V_R|=kn$, be a partition of $V$ such that $(u,v)\in E$ if and only if $u\in V_L,v\in V_R$ or $v\in V_L,u\in V_R$.
By following the proof arguments of Lemma 5.7 in \cite{long2014power}, one can show that
after the step 2 of the Swendsen-Wang dynamics starting from $X_0$ (and $Y_0$), there exists a constant $C$ such that the following event occurs with probability $1-O(1/n)$:
there are more than $Cn$
isolated vertices in both $V_L$, $V_R$.
Suppose the events happen from both $X_0$ and $Y_0$.
Then, we choose exactly $Cn$ isolated vertices in both $V_L, V_R$ (from $X_0,Y_0$)
and we consider the following coupling:
in the step 3 of the Swendsen-Wang dynamics starting from $X_0$ and $Y_0$, assign spins to components except for the chosen isolated vertices. 
Let $\hat X_1,\hat Y_1$ denote the spin configurations
except for the chosen isolated vertices. 
By applying the same arguments used for deriving \eqref{eq:p1markov}-\eqref{eq:giantbound}, we obtain
$$\|\boldsymbol\alpha(\hat X_1)-(\alpha_L^*-C/2,\alpha_R^*-C/2)\|_\infty,\|\boldsymbol\alpha(\hat Y_1)-(\alpha_L^*-C/2,\alpha_R^*-C/2)\|_\infty\le \frac{1}{2}L^\prime n^{-1/2}$$
for some constant $L^\prime$ with probability $\Theta(1)$.
Then it holds that
\begin{equation}\label{eq:fracbd}
\|\boldsymbol\alpha(\hat X_1)-\boldsymbol\alpha(\hat Y_1)\|_\infty\le L^\prime n^{-1/2}
\end{equation}
with probability $\Theta(1)$.
Assume that the event \eqref{eq:fracbd} occurs.
Now we show that there exists a coupling such that $\alpha_L(X_1)=\alpha_L(Y_1)$, $\alpha_R(X_1)=\alpha_R(Y_1)$ with probability $\Theta(1)$.
In this proof, we only provide a coupling such that $\alpha_L(X_1)=\alpha_L(Y_1)$ with probability $\Theta(1)$,
where one can easily extend the proof strategy to achieve $\alpha_R(X_1)=\alpha_R(Y_1)$. 

Now we provide a joint distribution on isolated vertices of $V_L$ in the step 3 of the Swendsen-Wang dynamics starting from $X_0$ and $Y_0$ so that $\alpha_L(X_1)=\alpha_L(Y_1)$ with probability $\Theta(1)$. 
Let $v_1,\dots,v_{Cn}$ denote the chosen isolated vertices without spin in $V_L$ for both chains.
For $1\le j\le Cn$, let define
$$Z_j=
\begin{cases}
&1\quad\text{if }X_1(v_j)=1\\
&0\quad\text{otherwise}
\end{cases}
\qquad
Z^\prime_j=
\begin{cases}
&1\quad\text{if }Y_1(v_j)=1\\
&0\quad\text{otherwise}
\end{cases}.$$
Let $Z=\sum_j Z_j,Z^\prime=\sum_j Z^\prime_j$.
Now we show that one can couple the spin configuration of $X_1$ and $Y_1$ with so that $\alpha_L(X_1)=\alpha_L(Y_1)$ (and also $\alpha_R(X_1)=\alpha_R(Y_1)$) with probability $\Theta(1)$
and complete the proof.
Consider $W\sim \text{Bin}(Cn, 1/2)$. Then, the distribution of $W$ is equivalent to the distribution of $Z$ (and $Z^\prime$).
Let define a coupling (joint distribution) $\mu$ on $Z,Z^\prime$ such that
$$\mu(Z=w,Z=w-\ell)=\min(\Pr(Z=w),\Pr(Z=w-\ell))$$
for $w\in\left[\frac{Cn}{2},\frac{Cn}{2}+L^\prime\sqrt{n}\right]$
where $|\ell:=n(\alpha_L(\hat X_1)-\alpha_L(\hat Y_1))|\le L^\prime\sqrt{n}$.
We remark that the construction of above coupling is equivalent to the coupling appears in Section 4.2 of \cite{levin2009markov}.
The coupling $\mu$ results that 
\begin{equation}\label{eq:couplingbd}
\begin{split}
\mu(Z=Z^\prime-\ell)&\ge\sum_{w\in\left[\frac{Cn}{2},\frac{Cn}{2}+L^\prime\sqrt{n}\right]}\mu(Z=w,Z^\prime=w-\ell).
\end{split}
\end{equation}
We now aim for showing that
\begin{equation}\label{eq:binombd}
\Pr(W=w)=\Omega(n^{-1/2})
\end{equation}
for all $w\in\left[\frac{Cn}{2}-L^\prime\sqrt{n},\frac{Cn}{2}+L^\prime\sqrt{n}\right]$, which leads
to $\mu(Z=Z^\prime-\ell)=\Theta(1)$ due to \eqref{eq:couplingbd}.

For $w\in\left[\frac{Cn}{2}-L^\prime\sqrt{n},\frac{Cn}{2}+L^\prime\sqrt{n}\right]$, it follows that
\begin{align*}
&\Pr\left(W=w\right)=\binom{Cn}{w}\left(\frac{1}{2}\right)^{Cn}\\
&\ge\binom{Cn}{\frac{Cn}{2}-L^\prime\sqrt{n}}\left(\frac{1}{2}\right)^{Cn}\\
&=\Theta(1)\frac{\sqrt{Cn}\left(\frac{Cn}{e}\right)^{Cn}}
{\sqrt{Cn-2L^\prime\sqrt{n}}\left(\frac{Cn-2L^\prime\sqrt{n}}{2e}\right)^{\frac{Cn-2L^\prime\sqrt{n}}{2}}\sqrt{Cn+2L^\prime\sqrt{n}}\left(\frac{Cn+2L^\prime\sqrt{n}}{2e}\right)^{\frac{Cn+2L^\prime\sqrt{n}}{2}}}\left(\frac{1}{2}\right)^{Cn}\\
&=\Theta(n^{-1/2})\frac{(Cn)^n}{(Cn-2L^\prime\sqrt{n})^{\frac{Cn-2L^\prime\sqrt{n}}{2}}(Cn+2L^\prime\sqrt{n})^{\frac{Cn+2L^\prime\sqrt{n}}{2}}}\\
&=\Theta(n^{-1/2})\frac{1}{\left(1-\frac{2L^\prime\sqrt{n}}{Cn}\right)^{\frac{Cn-2L^\prime\sqrt{n}}{2}}
\left(1+\frac{2L^\prime\sqrt{n}}{Cn}\right)^{\frac{Cn+2L^\prime\sqrt{n}}{2}}}\\
&\ge\Theta(n^{-1/2})\frac{1}{e^{\frac{4L^{\prime2}}{C}}}\\
&=\Theta(n^{-1/2})
\end{align*}
where the second equality follows from Stirling's formula.
By combining \eqref{eq:couplingbd} and \eqref{eq:binombd}, we obtain
$$\mu(Z=Z^\prime-\ell)=\Theta(1)$$
and therefore there exists a coupling on $(X_t,Y_t)$ such that $\alpha_L(X_1)=\alpha_L(Y_1)$ with probability $\Theta(1)$.
This completes the proof of Lemma \ref{lem:p1coupling1}.

\subsection{Proof of Lemma \ref{lem:p2close}}\label{sec:p2close}
From Lemma \ref{lem:maxprob}, we know that $(\alpha_L^*,\alpha_R^*)=(1/2,1/2)$.
Throughout this proof, we use $(1/2,1/2)$ instead of $(\alpha_L^*,\alpha_R^*)$.
First, choose a constant $\delta>0$ small enough so that $F(1/2+\delta,1/2+\delta)=(1/2,1/2)$, i.e. $(1/2+\delta,1/2+\delta)$ is subcritical.
Then, from Lemma \ref{lem:convergence}, there exists a constant $T$ such that $F^{(T)}([0,1])\le(1/2+\delta/2,1/2+\delta/2)$.
One can directly notice that that within $T$ iterations of the Swendsen-Wang chain, the size of the larger spin class becomes less than $(1/2+\delta,1/2+\delta)$ with probability $1-o(1)$ by Lemma \ref{lem:SWF}.
Furthermore, since $(1/2+\delta,1/2+\delta)$ is subcritical,
in the step 2 of the Swendsen-Wang dynamics
at the next iteration, the larger spin class becomes subcritical,
i.e. $\boldsymbol\alpha(X_{T+1})=(1/2+o(1),1/2+o(1))$ with probability $1-o(1)$ by Lemma \ref{lem:SWF}.
Given the event $\boldsymbol\alpha(X_{T+1})=(1/2+o(1),1/2+o(1))$,
after the step 2 of the Swendsen-Wang dynamics starting from $X_{T+1}$ satisfies the following:
$$E\bigg[\sum_{i\ge 1} |C_i|^2\bigg]=O(n),$$
where we use Lemma \ref{lem:concentration1} \ref{it:wsxb12a}.
By applying the same arguments used for deriving \eqref{eq:p1markov} and \eqref{eq:p1azuma}, we have
$$X_{T+2}=(1/2+O(n^{-1/2}),1/2+O(n^{-1/2})),\qquad\mbox{with probability $\Theta(1)$.}$$
This completes the proof of Lemma \ref{lem:p2close}.

\section{Proofs of Technical Lemmas}
\subsection{Proof of Lemma \ref{lem:attractive}}\label{sec:pflem:attractive}
In this proof, we first show that $F$ has the unique fixed point $(1/2,1/2)$ for $B<2$ and
$(\alpha_L^*,\alpha_R^*)$ for $B>2$.
Before starting the proof, we note that $\alpha_L<1/2,\alpha_R>1/2$ (or $\alpha_L>1/2,\alpha_R<1/2$) cannot be a solution of \eqref{eq:fixedF}.
To help the proof, we use the substitution $z_L=2\alpha_L-1$ and $z_R=2\alpha_R-1$.
By substituting $z_L,z_R$ into \eqref{eq:fixedF}, we have
\begin{equation}\label{eq:probstate2}
z_L=\frac{\sqrt{k}}{B}\log\frac{1+z_R}{1-z_R}\qquad\qquad
z_R=\frac{1}{B\sqrt{k}}\log\frac{1+z_L}{1-z_L},
\end{equation}
i.e. any fixed point of $F$ must satisfies \eqref{eq:probstate2}.
First, consider the case that $B<2$. 
One can easily check that $(1/2,1/2)$ is a fixed point of $F$ and $\alpha_L,\alpha_R<1/2$ cannot be a fixed point of $F$.
Now, suppose that there exists a solution $z_L,z_R>0$ of \eqref{eq:probstate2},
i.e. there exists $\alpha_L,\alpha_R>1/2$ satisfying \eqref{eq:fixedF}.
Using the inequality $\log\frac{1+x}{1-x}>2x$ for $x>0$ and \eqref{eq:probstate2}, we have
$$z_L>\frac{4}{B^2}z_L\qquad\qquad z_R>\frac{4}{B^2}z_R.$$
Since we assumed that $B<2$, the above inequalities leads to contradiction and results that $(1/2,1/2)$ is the only
fixed point of $F$ for $B<2$.
Now, consider the case that $B>2$. 
We first define functions $g(x),y(x)$ as below:
\begin{equation*}
y(x):=\frac{1}{B\sqrt{k}}\log\frac{1+x}{1-x}\qquad\qquad
g(x):=\frac{\sqrt{k}}{B}\log\frac{1+y(x)}{1-y(x)}.
\end{equation*}
Then $x$ is a fixed point of $g$ if and only if $(z_L,z_R)=(x,y(x))$ is a solution of \eqref{eq:probstate2}.
Now we show that there exists the unique fixed point $x>0$ of $g$.
Suppose there exist two fixed points $x_1,x_2$ of $g$.
By mean value theorem, there exists $x^\prime$ between $x_1,x_2$ such that
$\frac{dg}{dx}(x^\prime)=1$.
However, the derivative of $g(x)$ with respect to $x$ is
$$\frac{dg}{dx}(x)=\frac{4 k}{1-x^2} \frac{1}{B^2 k - \log^2 \frac{1+x}{1-x} }$$
and at $x=x^\prime$ we have
\begin{equation}\label{eq:dg}
\frac{4 k}{1-x^{2}} ={B^2 k - \log^2 \frac{1+x}{1-x} }.
\end{equation}
One can observe that LHS of \eqref{eq:dg} is increasing with $x$ but
RHS of \eqref{eq:dg} is decreasing with $x$, 
i.e. there are at most two fixed points of $g$ and therefore there are at most
two solutions of \eqref{eq:fixedF}.
$(1/2,1/2)$ is a solution of \eqref{eq:fixedF} but it is not a fixed point of $F$.
However, since $F:[0,1]^2\rightarrow[0,1]^2$ and $F$ is continuous, by Brouwer's fixed point theorem, $F$ has a fixed point.
Furthermore, for $(\alpha_L,\alpha_R)\le(1/2,1/2)$, we have $F(\alpha_L,\alpha_R)\ge(1/2,1/2)$.
Using this facts, one can conclude that $F$ has a unique fixed point
$(\alpha_L^*,\alpha_R^*)>(1/2,1/2)$ for $B>2$.

Now, we show that the fixed point of $F$ is Jacobian attractive.
Consider the Jacobian $D(F)$ of $F$, given by
\begin{align}
D(F)&=
\begin{pmatrix}
\dfrac{\partial F_L}{\partial\alpha_L}&\dfrac{\partial F_L}{\partial\alpha_R}\\
&\\
\dfrac{\partial F_R}{\partial\alpha_L}&\dfrac{\partial F_R}{\partial\alpha_R}
\end{pmatrix}\notag\\
&=\frac{1}{2}\frac{1}{1-(1-\theta_L)(1-\theta_R)B^2\alpha_L\alpha_R}
\begin{pmatrix}
\theta_L&(1-\theta_L)\theta_RB\sqrt{k}\alpha_L\\
&\\
(1-\theta_R)\theta_LB\alpha_R/\sqrt{k}&\theta_R\label{eq:jacobian}
\end{pmatrix}
\end{align}
where $(\theta_L,\theta_R)$ is a solution of \eqref{eq:thetaF}.
For $B<2$, $D(F)$ is a zero matrix at (1/2,1/2), i.e. the largest eigen value of $D(F)$ is zero.
Therefore the fixed point of $F$ is Jacobian attractive for $B<2$.
Suppose $B>2$.
Using \eqref{eq:thetaF} and by direct calculation of the largest eigenvalue, the largest eigenvalue $\lambda$ of $D(F)$ can be bounded as below:
\begin{equation}\label{eq:eigen1}
|\lambda|<\frac{1}{2}\frac{\theta_L+\theta_R}{1-\frac{(1-\theta_L)(1-\theta_R)}{\theta_L\theta_R}\log(1-\theta_L)\log(1-\theta_R)}.
\end{equation}
Since we are interested in $\lambda$ at the fixed point, we only need to consider $\theta_L,\theta_R>0$.
Now we show that RHS of \eqref{eq:eigen1} is strictly smaller than $1$
to prove that $F$ is Jacobian attractive at $(\alpha_L^*,\alpha_R^*)$.
Consider the following function $h$
$$h(\theta_L,\theta_R):=2-\theta_L-\theta_R-
2\frac{(1-\theta_L)(1-\theta_R)}{\theta_L\theta_R}\log(1-\theta_L)\log(1-\theta_R).$$
One can notice that $h(\theta_L,\theta_R)>0$ if and only if RHS of \eqref{eq:eigen1} is strictly smaller than $1$.
We bound $h$ using the following claim.
\begin{claim}\label{clm:goodbd}
For $0<x<1$, the following inequality holds:
$$-\frac{1-x}{x}\log(1-x)<\sqrt{1-x}.$$
\end{claim}
\begin{proof}
Let 
$$f(x):=\frac{\sqrt{1-x}}{x}\log(1-x).$$
It suffices to show that $-1<f(x)$ for $0<x<1$.
We have $\lim_{x\rightarrow 0^+}f(x)=-1$.
Furthermore, $f$ is strictly increasing for $0<x<1$ since
$$\frac{df}{dx}(x)=-\frac{2-x}{2x^2\sqrt{1-x}}\log(1-x)-\frac{1}{x\sqrt{1-x}}>0,$$
where the last inequality can be verified by using the Taylor series of $\log(1-x)$.
This implies that $-1<f(x)$ for $0<x<1$, completing the proof of Claim \ref{clm:goodbd}.
\end{proof}
Using Claim \ref{clm:goodbd}, we have
$$h(\theta_L,\theta_R)>2-\theta_L-\theta_R-2\sqrt{(1-\theta_L)(1-\theta_R)}\ge 0$$
for $0<\theta_L,\theta_R<1$.
This implies that RHS of \eqref{eq:eigen1} is strictly smaller than $1$ and therefore $|\lambda|<1$, i.e. $(\alpha_L^*,\alpha_R^*)$ is Jacobian attractive fixedpoint of $F$ for $B>2$.
This completes the proof of Lemma \ref{lem:attractive}.

\subsection{Proof of Lemma \ref{lem:convergence}}\label{sec:pflem:convergence}

In this proof, we first show that $F$ is monotonically increasing function.
From the formulation \eqref{eq:jacobian} of the Jacobian of $F$, every entries of $D(F)$ is non-negative, i.e. $F$ is monotonically increasing, if and only if the following inequality holds
\begin{equation}\label{eq:denom}
1-(1-\theta_L)(1-\theta_R)B^2\alpha_L\alpha_R>0.
\end{equation}
Since $\theta_L=\theta_R=0$ if and only if $\sqrt{\alpha_L\alpha_R}B\le 1$, we only need to consider the case that $\sqrt{\alpha_L\alpha_R}B>1$ (we can ignore the case $\sqrt{\alpha_L\alpha_R}=1$ for proving that $F$ is monotonically increasing as $F$ is continuous).
Using \eqref{eq:thetaF}, LHS of \eqref{eq:denom} can be represented as
$$1-\frac{(1-\theta_L)(1-\theta_R)}{\theta_L\theta_R}\log(1-\theta_L)\log(1-\theta_R).$$
By Claim \ref{clm:goodbd}, we have
$$1-\frac{(1-\theta_L)(1-\theta_R)}{\theta_L\theta_R}\log(1-\theta_L)\log(1-\theta_R)>1-\sqrt{(1-\theta_L)(1-\theta_R)}>0$$
for $0<\theta_L,\theta_R<1$.
This results that $F$ is monotonically increasing.

Since $F$ is monotonically increasing, $F^{(t)}(0,0)\le F^{(t)}(\alpha_L,\alpha_R)\le F^{(t)}(1,1)$ for any $(\alpha_L,\alpha_R)$, i.e. it is enough to show that sequences $[F^{(t)}(0,0)]_t$ and $[F^{(t)}(1,1)]_t$ converge to the fixed point of $F$.
Let $(\alpha_L^*,\alpha_R^*)$ be the fixed point of $F$.
From the definition of $F$, we have $F(1,1)\le(1,1)$.
Using the monotonicity, we have $F^{(2)}(1,1)\le F(1,1)$.
By applying this argument repeatedly, one can argue that $[F^{(t)}(1,1)]_t$ is a decreasing sequence
and bounded below by the fixed point of $F$.
By the monotone convergence theorem and lemma \ref{lem:attractive}, $[F^{(t)}(1,1)]_t$ converges to the fixed point of $F$.
Similarly $[F^{(t)}(0,0)]_t$ converges to the fixed point of $F$.
This completes the proof of Lemma \ref{lem:convergence}.

\subsection{Proof of Lemma \ref{lem:maxprob}}\label{sec:pflem:maxprob}
We first formulate the probability that a phase $(\alpha_L,\alpha_R)$ occurs.
This probability can be formulated as follows:
\begin{align*}
\Pr(\alpha_L,\alpha_R)&\propto\binom{n}{\alpha_L n}\binom{kn}{\alpha_R kn}\left(1-\frac{B}{n\sqrt{k}}\right)^{kn^2(\alpha_L(1-\alpha_R)+\alpha_R(1-\alpha_L))}\\
&\approx\frac{1}{2\pi n\sqrt{\alpha_L(1-\alpha_L)\alpha_R(1-\alpha_R)k}}\alpha_L^{-\alpha_Ln}(1-\alpha_L)^{-(1-\alpha_L)n}\\
&\qquad\quad\times\alpha_R^{-\alpha_Rkn}(1-\alpha_R)^{-(1-\alpha_R)kn} \exp\left({-Bn\sqrt{k}(\alpha_L(1-\alpha_R)+\alpha_R(1-\alpha_L))}\right)\\
&=\frac{1}{2\pi n\sqrt{\alpha_L(1-\alpha_L)\alpha_R(1-\alpha_R)k}}\exp\left({n\sqrt{k}\psi(\alpha_L,\alpha_R)}\right)
\end{align*}
where we use Stirling's formula for the second line and $\psi$ is defined as
\begin{align*}
\psi(\alpha_L,\alpha_R):=&-B(\alpha_L+\alpha_R-2\alpha_L\alpha_R)-\frac{\alpha_L}{\sqrt{k}}\log\alpha_L-\frac{1-\alpha_L}{\sqrt{k}}\log(1-\alpha_L)\\
&\qquad-\sqrt{k}\alpha_R\log\alpha_R-\sqrt{k}(1-\alpha_R)\log(1-\alpha_R).
\end{align*}
Since $\psi(\alpha_L,\alpha_R)$ determines the exponential order of the probability of the phase $(\alpha_L,\alpha_R)$ and the number of possible phases is bounded by a polynomial in $n$,  the maximum a posteriori phase of the Ising model is asymptotically given by  the phase which achieves the maximum value of $\psi$. 

Now we analyze the phase $(\alpha_L,\alpha_R)$ maximizing $\psi$.
By taking partial derivative of $\psi$ with respect to $\alpha_L$ and $\alpha_R$, we have
\begin{align*}
\frac{\partial\psi(\alpha_L,\alpha_R)}{\partial\alpha_L}&=-B(1-2\alpha_R)-\frac{1}{\sqrt{k}}\log\alpha_L+\frac{1}{\sqrt{k}}\log(1-\alpha_L)\\
\frac{\partial\psi(\alpha_L,\alpha_R)}{\partial\alpha_R}&=-B(1-2\alpha_L)-{\sqrt{k}}\log\alpha_R+{\sqrt{k}}\log(1-\alpha_R).
\end{align*}
By simple calculation, one can check that $\frac{\partial\psi(\alpha_L,\alpha_R)}{\partial\alpha_L}=\frac{\partial\psi(\alpha_L,\alpha_R)}{\partial\alpha_R}=0$
if and only if the following relation holds
\begin{equation}\label{eq:probstate}
\exp\left({B\sqrt{k}(1-2\alpha_R)}\right)=\frac{1-\alpha_L}{\alpha_L}\qquad\qquad
\exp\left({\frac{B}{\sqrt{k}}(1-2\alpha_L)}\right)=\frac{1-\alpha_R}{\alpha_R}
\end{equation}
which is equivalent to \eqref{eq:fixedF}.
One can easily check that $\alpha_L=\alpha_R=1/2$ is a solution of \eqref{eq:probstate}.
If $(\alpha_L,\alpha_R)$ is a solution of \eqref{eq:probstate},
then $(1-\alpha_L,1-\alpha_R)$ is a solution of \eqref{eq:probstate}.
Furthermore, LHS and RHS of the first (and the second) equation of \eqref{eq:probstate} are decreasing with respect to $\alpha_R,\alpha_L$ (and $\alpha_L,\alpha_R$) respectively.
Since $(1/2,1/2)$ is a solution of \eqref{eq:probstate}, any solution $(\alpha_L,\alpha_R)$ of \eqref{eq:probstate} satisfies $\alpha_L,\alpha_R\ge 1/2$ or
$\alpha_L,\alpha_R\le 1/2$.
Therefore, we only consider critical points of $\psi$ in $[1/2,1]^2$.
In the proof of Lemma \ref{lem:attractive} we have shown that
\eqref{eq:probstate} has the only solution $(1/2,1/2)$ for $B\le 2$ and
\eqref{eq:probstate} has only two solutions $(1/2,1/2)$, $(\alpha_L^*,\alpha_R^*)$ for $B>2$.
Now we show that $(1/2,1/2)$, $(\alpha_L^*,\alpha_R^*)$ achieve the maximum value of $\psi$ for $B\le 2,B>2$ respectively by showing that the Hessian of $\psi$ is negative semidefinite at $(1/2,1/2)$, $(\alpha_L^*,\alpha_R^*)$ for $B\le 2,B>2$ respectively.
The hessian $H(\psi)$ of $\psi$ is as follows
$$H(\psi)=\begin{pmatrix}
-\dfrac{1}{\alpha_L(1-\alpha_L)\sqrt{k}}&2B\\
2B&-\dfrac{\sqrt{k}}{\alpha_R(1-\alpha_R)}
\end{pmatrix}.$$
By simple calculations, one can check that $H(\psi)$ is negative semidefinite
if and only if
\begin{equation}\label{eq:nsd}
2B\le\sqrt{\frac{1}{\alpha_L(1-\alpha_L)\alpha_R(1-\alpha_R)}}.
\end{equation}
Since \eqref{eq:nsd} holds for any $B\le 2$, $(1/2,1/2)$ maximizes $\psi$.

Now we show that $(\alpha_L^*,\alpha_R^*)$ maximizes $\psi$ for $B>2$.
Consider $H(\psi)$ at $(1/2,1/2)$ and $(\alpha_L^*,\alpha_R^*)$.
$H(\psi)$ is negative semidefinite if and only if \eqref{eq:nsd} holds.
However, $(1/2,1/2)$ does not satisfies \eqref{eq:nsd} and therefore $(1/2,1/2)$ is not a local maximum of $F$.
Let $z_L^*=2\alpha_L^*-1$ and $z_R^*=2\alpha_R^*-1$.
Then \eqref{eq:nsd} at $(\alpha_L^*,\alpha_R^*)$ is equivalent to
\begin{equation}\label{eq:nsd2}
\frac{1}{4}\frac{(1-z_L^{*2})(1-z_R^{*2})}{z_L^*z_R^*}\log\frac{1+z_L^*}{1-z_L^*}\log\frac{1+z_R^*}{1-z_R^*}\le 1
\end{equation}
where we additionally use the fact that $(z_L^*,z_R^*)$ is a solution of \eqref{eq:probstate2}.
Let define $h(x):=\frac{1-x^2}{x}\log\frac{1+x}{1-x}$.
We have $\lim_{x\rightarrow 0^+}h(x)=2$.
The derivative of $h$ is strictly negative as
\begin{equation}\label{eq:nsd3}
    \frac{dh}{dx}(x)=-\frac{1+x^2}{x^2}\log\frac{1+x}{1-x}\frac{2}{x}<-2x\le0
\end{equation}
where we use an inequality $\log\frac{1+x}{1-x}>2x$ for $x>0$.
Since \eqref{eq:nsd3} and $\lim_{x\rightarrow 0^+}h(x)=2$ implies 
\begin{equation}\label{eq:nsd4}
\frac{1}{4}\frac{(1-z_L^{*2})(1-z_R^{*2})}{z_L^*z_R^*}\log\frac{1+z_L^*}{1-z_L^*}\log\frac{1+z_R^*}{1-z_R^*}<1
\end{equation}
and this implies \eqref{eq:nsd2}, $(\alpha_L^*,\alpha_R^*)$ is the only local maximum of $\psi$ on $[1/2,1]^2$ for $B>2$.
Recall that every local maximum point $(\alpha_L,\alpha_R)$ of $\psi$ satisfies that $\alpha_L,\alpha_R\ge 1/2$ or $\alpha_L,\alpha_R\le 1/2$.
This implies that $(\alpha^*_L,\alpha^*_R),(1-\alpha^*_L,1-\alpha^*_R)$ are only local maxima of $\psi$, i.e. $(1-\alpha^*_L,1-\alpha^*_R)$ achieves maximum of $\psi$ in $[0,1/2]\times[0,1]$ and $(\alpha^*_L,\alpha^*_R)$ achieves maximum of $\psi$ in $[1/2,1]\times[0,1]$ for $B>2$.
By Using this, one can conclude that $(\alpha^*_L,\alpha^*_R),(1-\alpha^*_L,1-\alpha^*_R)$ achieve the maximum of $\psi$ for $B>2$.
This completes the proof of Lemma \ref{lem:maxprob}.

\subsection{Proof of Lemma \ref{lem:concentration1}}\label{sec:pflem:concentration1}
In this section, we prove Lemma \ref{lem:concentration1}.

We first prove part \ref{it:wsxb12a}. 
In order to bound the component sizes of $G$, we consider the following branching process 
to explore a connected component of the bipartite Erd\H{o}s-R\'enyi random graph $G=(V_L,V_R,E)\sim G(n,kn,p)$ with $p=\frac{B}{n\sqrt{k}}$.
\begin{itemize}
\item[1.] Set $t=0$. Choose $u_0\in V_L$ and initialize $S_L=S_R=\emptyset$, $W_L={u_0},W_R=\emptyset$.
\item[2.] Set $t\leftarrow t+1$.
Choose $u_i\in W_L$ and choose random neighbors $v_1,\dots v_{r_i}$ of $u_i$ from $V_R-S_R-W_R$ where each neighbor of $u_i$ is chosen with probability $\frac{B}{n\sqrt{k}}$. 
Set $W_R=W_R\cup\{v_1,\dots v_{r_i}\}$, $W_L=W_L-\{u_i\}$ and $S_L=S_L\cup\{u_i\}$.
\item[3.] For each $v_j\in W_R$, choose random neighbors $u_{j1},\dots,u_{js_j}$ of $v_j$ from $V_L-S_L-W_L$ where each neighbor of $v_j$ is chosen with probability $\frac{B}{n\sqrt{k}}$.
Set $W_L=W_L\cup\{u_{j1},\dots,u_{js_j}\}$, $W_R=W_R-\{v_j\}$ and $S_R=S_R\cup\{v_j\}$.
Repeat the step 3 until $W_R=\emptyset$.
\item[4.] Repeat steps 2-3 until $W_L\cup W_R=\emptyset$.
\end{itemize}
For each $t$-th iteration, let define a random variable $K_t:=|W_L|$ at the beginning of the step 4 of the branching process.
Then the stopping time $\arg\min_t (K_t=0)$ decides the number of vertices in $V_L$ in the component of $G(n,kn,p)$ containing $u_0$.
One can observe that $K_t$ is bounded above by the random variable $(\sum_{i=0}^t R_i)-t$ where $R_0=1$
and $R_i\sim\text{Bin}\left(\text{Bin}\left(n,\frac{B}{n\sqrt{k}}\right)kn,\frac{B}{n\sqrt{k}}\right)$. 
Similarly, one can construct the branching process starting from $u_0\in V_R$ and define $K_t^\prime$ as $K_t$. Then $K_t^\prime$ is bounded above by $(\sum_{i=0}^t R^\prime_i)-t$ where $R^\prime_0=1$
, $R_i^\prime\sim\text{Bin}\left(\text{Bin}\left(kn,\frac{B}{n\sqrt{k}}\right)n,\frac{B}{n\sqrt{k}}\right)$.

Let $C(v)$ be the component of $G$ containing $v$.
Observe that 
$$E\bigg[\sum_{i\ge 1}|C_i|^2\bigg]=E\bigg[\sum_{v\in V_L\cup V_R}|C(v)|\bigg]=(1+k)nE\big[|C(v)|\big].$$
To complete the proof, it thus suffices to  show that $E\left[|C(v)|\right]=O(1)$.
Define the following stopping times $\tau,\tau^\prime$:
$$\tau:=\arg\min_t \left(\left(\sum_{i=0}^t R_i\right)-t=0\right)\qquad\tau^\prime:=\arg\min_t \left(\left(\sum_{i=0}^t R_i^\prime\right)-t=0\right).$$
Since $K_t,K_t^\prime$ are bounded above by $(\sum_{i=0}^t R_i)-t$, $(\sum_{i=0}^t R^\prime_i)-t$ respectively, we have
$$E\big[|C(v)|\big]\le E[\tau]+E[\tau^\prime].$$
By applying Wald's lemma, we can conclude that
$$E[\tau],E[\tau^\prime]=O(1)$$
and this completes the proof of part \ref{it:wsxb12a} of Lemma \ref{lem:concentration1}.

Next, we prove part \ref{it:wsxb12b} of Lemma \ref{lem:concentration1}.
We first prove the following.
\begin{claim}\label{clm:giantsub}
For $B>1$, $(1-\theta_L)(1-\theta_R)B^2<1$ where $\theta_L,\theta_R$ are solution of \eqref{eq:theta4}, i.e. the rest part except for the giant component is subcritical.
\end{claim}
\begin{proof}
Using \eqref{eq:theta4}, $(1-\theta_L)(1-\theta_R)B^2$ reduces to
$$(1-\theta_L)(1-\theta_R)B^2=\frac{(1-\theta_L)(1-\theta_R)}{\theta_L\theta_R}\log(1-\theta_L)\log(1-\theta_R).$$
By applying Claim \ref{clm:goodbd} to the RHS of the above identity, we completes the proof of Claim \ref{clm:giantsub}.
\end{proof}
By Claim \ref{clm:giantsub}, we know that the induced subgraph of vertices which are not in the giant component, $C_1$, is subcritical.
Let $\varepsilon>0$ be a small enough constant which satisfies that 
\begin{equation}\label{eq:giantsub}
(1-\theta_L+\varepsilon)(1-\theta_R+\varepsilon)B^2<1.
\end{equation}
By following the proof of Theorem 9 of \cite{johansson2012giant} and applying Azuma's inequality, one can conclude that 
\begin{align*}
&\Pr\big(|C_1\cap V_L|<(\theta_L-\varepsilon)n\big)<e^{-\Omega(n)}\\
&\Pr\big(|C_1\cap V_R|<(\theta_R-\varepsilon)kn\big)<e^{-\Omega(n)}
\end{align*}
for some constant $c$.
Let $\mathcal{E}$ be the event that $|C_1\cap V_L|>(\theta_L-\varepsilon)n,|C_1\cap V_R|>(\theta_R-\varepsilon)kn$. By \eqref{eq:giantsub} and part \ref{it:wsxb12a} of Lemma \ref{lem:concentration1}, we have
$$E\bigg[\sum_{i\ge 2}|C_i|^2\,\Big|\,\mathcal{E}\bigg]=O(n).$$
Since $\Pr(\mathcal{E})=1-e^{-\Omega(n)}$, removing the conditioning on the event $\mathcal{E}$ can only affect the bound by $o(1)$. This yields part  \ref{it:wsxb12b} of Lemma \ref{lem:concentration1}, and thus completes the proof.

\subsection{Proof of Lemma \ref{lem:normal}}\label{sec:pflem:normal}
Call $v\in V_L\cup V_R$ `small' if $v$ is not in the giant component.
For each $v\in V_L\cup V_R$, let $S_{v}$ be the indicator random variable that $v$ is small.
Define $S_L:=\sum_{v\in V_L}S_{v}$ and $S_R:=\sum_{v\in V_R}S_{v}$.
From Lemma \ref{thm:giant}, we know that 
$$\mbox{ for all $v\in V_L$, } \Pr(S_{v}=1)=1-\theta_L.$$
To bound the variance of the giant component, our goal is to bound the variance of $S_L$, $S_R$.
We bound the second moment of $S_L$ as below:
\begin{align*}
E\left[S_L^2\right]&=\sum_{v\in V_L}E\left[S_{v}^2\right]+\sum_{\substack{u\ne v\\u,v\in V_L}}E[S_{u}S_{v}]\\
&=E[S_L]+\sum_{\substack{u\ne v\\u,v\in V_L}}\Pr(\text{$u,v$ are small})\\
&=E[S_L]+\sum_{v\in V_L}\Pr(\text{$v$ is small})\sum_{\substack{u\ne v\\u\in V_L}}\Pr(\text{$u$ is small}\,|\,\text{$v$ is small})
\end{align*}
Note that for each $v\in V_L$, we have that
\begin{align*}
&\sum_{\substack{u\ne v\\u\in V_L}}\Pr(\text{$u$ is small}\,|\,\text{$v$ is small})\\
&=\sum_{\substack{u\ne v:~u\in V_L\\\text{$u,v$ are in same}\\\text{component}}}\Pr(\text{$u$ is small}\,|\,\text{$v$ is small})+\sum_{\substack{u\ne v:~u\in V_L\\\text{$u,v$ are in different}\\\text{components}}}\Pr(\text{$u$ is small}\,|\,\text{$v$ is small}).
\end{align*}
However, we have
\begin{align}\label{eq:samecomp}
\sum_{v\in V_L}\Pr(\text{$v$ is small})\sum_{\substack{u\ne v:~u\in V_L\\\text{$u,v$ are in same}\\\text{component}}}\Pr(\text{$u$ is small}\,|\,\text{$v$ is small})\notag
&\le\sum_{v\in V_L}(|C(v)|-1)\\
&\le\sum_{i\ge 2}|C_i|(|C_i|-1)\\
&\le wKn\notag
\end{align}
where $C(v)$ is a component containing a small vertex $v$.
The last inequality of \eqref{eq:samecomp} follows from the assumption $\sum_{i\ge 2}|C_i|^2< wKn$.
For $u,v$ which are in different components, asymptotically we have
\begin{equation}\label{eq:diffcomp}
\Pr(\text{$u$ is small}\,|\,\text{$v$ is small})=\Pr(\text{$u$ is small})=1-\theta_L
\end{equation}
as $|C(v)|=O(\log^2 n)$ for small vertex $v$ by Lemma \ref{thm:giant}.
Combining \eqref{eq:samecomp} and \eqref{eq:diffcomp} results
\begin{align}
E[S_L^2]&=E[S_L]+\sum_{v\in V_L}\Pr(\text{$v$ is small})\sum_{u\ne v}\Pr(\text{$u$ is small}\,|\,\text{$v$ is small})\notag\\
&\le(1-\theta_L)n+wKn+(1-\theta_L)^2n^2.\label{eq:varbound}
\end{align}
\eqref{eq:varbound} directly leads to
$$\text{Var}(S_L)\le(1-\theta_L+wK)n.$$
Using Chebyshev's inequality, we bound the deviation of $S_L$ from its expectation as
\begin{equation}\label{eq:chebl}
\Pr(|S_L-(1-\theta_L)n|\ge w\sqrt{n})
\le\frac{1-\theta_L+wK}{w^2}.
\end{equation}
One can apply the similar argument for $V_R$ and achieve
\begin{equation}\label{eq:chebr}
\Pr(|S_R-(1-\theta_R)kn|\ge w\sqrt{kn})
\le\frac{(1-\theta_R)k+wK}{kw^2}.
\end{equation}
Combining \eqref{eq:chebl}, \eqref{eq:chebr} results
$$\Pr(\{|S_L-(1-\theta_L)n|\ge w\sqrt{n}\}\cup
\{|S_R-(1-\theta_R)kn|\ge w\sqrt{n}\})
\le\frac{2K}{w}+\frac{1+k}{w^2}.$$
This completes the proof of Lemma \ref{lem:normal}.

\end{document}